\documentclass[final]{cvpr}

\usepackage{times}
\usepackage{epsfig}
\usepackage{graphicx}
\usepackage{amsmath}
\usepackage{amssymb}

\usepackage{makecell}
\usepackage{mathtools}
\usepackage{adjustbox}

\usepackage{cmap}
\usepackage[T1]{fontenc}
\usepackage{graphicx}
\usepackage{amsfonts}
\usepackage{times}
\usepackage[numbers]{natbib}
\usepackage{multicol}
\usepackage{url}
\usepackage[noend,ruled,vlined,linesnumbered]{algorithm2e}
\usepackage[font={it}]{caption}
\usepackage{subcaption}
\usepackage{amsthm}
\usepackage{graphicx,wrapfig}
\usepackage{etoolbox}
\usepackage{tocloft}
\usepackage{array}
\usepackage{enumitem}

\newcommand{\abs}[1]{\left| #1\right|}
\newcommand{\norm}[1]{\left\lVert#1\right\rVert}

\newcommand{\br}[1]{\left\{#1\right\}}

\newcommand{\REAL}{\ensuremath{\mathbb{R}}}

\newtheorem{theorem}{Theorem}
\newtheorem{lemma}[theorem]{Lemma}
\newtheorem{observation}[theorem]{Observation}
\newcommand{\iters}{\beta}

\newtheorem{definition}[theorem]{Definition}
\newtheorem{corollary}[theorem]{Corollary}
\newtheorem{claim}{Claim}[theorem]
\newcommand{\dist}{\mathrm{dist}}
\newcommand{\smallest}{\mathrm{smallest}}

\newcommand{\M}{m}
\newcommand{\D}{\mathrm{D}}

\newcommand{\I}{S}
\newcommand{\rot}{\textsc{SO}}
\newcommand{\idmat}{ \textsc{I}}

\newcommand{\NN}{\textsc{NN}}

\newcommand{\linspan}{\mathrm{sp}}
\newcommand{\R}{\mathcal{R}}
\newcommand{\lip}{\ell}
\newcommand{\cost}{\mathrm{cost}}
\newcommand{\proj}{\mathrm{proj}}

\newcommand{\OPT}{\mathrm{OPT}}

\newcommand{\algnamematching}{\textsc{Align-and-Match}}
\newcommand{\algnamegetrot}{\textsc{Get-Rot}}
\newcommand{\algNameProbRot}{\textsc{Prob-Align}}
\newcommand{\algNameApproxAlignment}{\textsc{Prob-Alignment}}
\newcommand{\algname}{\textsc{Approx-Alignment}}

\renewcommand{\paragraph}[1]{\medskip\noindent\textbf{{#1} }}
\newcommand{\alignments}{\textsc{Alignments}}

\renewcommand{\Pi}{\Psi}
\newcommand{\sdim}{\tau}
\newcommand{\jdim}{\tau}

\newcommand{\matchAlg}{\texttt{P-ICP}}
\newcommand{\matchAlgICP}{\texttt{P-ICP-Refined}}
\newcommand{\ICP}{\texttt{ICP}}
\newcommand{\GOICP}{\texttt{GO-ICP}}

\newcommand{\CPD}{\texttt{CPD}}

\DeclareMathOperator*{\argmin}{arg\,min}

\newif\ifinclude
\includetrue

\newif\ifproofs
\proofsfalse

\usepackage[pagebackref=true,breaklinks=true,colorlinks,bookmarks=false]{hyperref}

\def\cvprPaperID{8533} 
\def\confYear{CVPR 2021}

\begin{document}

\title{Provably Approximated ICP}


\author{
    Ibrahim Jubran$^1$, Alaa Maalouf$^1$, Ron Kimmel$^2$, Dan Feldman$^1$\\
    $^1$University of Haifa, $^2$Technion\\
    {\tt\small \{ibrahim.jub,  alaamalouf12, dannyf.post\}@gmail.com}\\
    {\tt\small \{ron\}@cs.technion.ac.il}
    }

\maketitle


\begin{abstract}
The goal of the \emph{alignment problem} is to align a (given) point cloud $P = \{p_1,\cdots,p_n\}$ to another (observed) point cloud $Q = \{q_1,\cdots,q_n\}$. That is, to compute a rotation matrix $R \in \mathbb{R}^{3 \times 3}$ and a translation vector $t \in \mathbb{R}^{3}$ that minimize the sum of paired distances $\sum_{i=1}^n D(Rp_i-t,q_i)$ for some distance function $D$.
A harder version is the \emph{registration problem}, where the correspondence is unknown, and the minimum is also over all possible correspondence functions from $P$ to $Q$.
Heuristics such as the Iterative Closest Point (ICP) algorithm and its variants were suggested for these problems, but none yield a provable non-trivial approximation for the global optimum.

We prove that there \emph{always} exists a ``witness'' set of $3$ pairs in $P \times Q$ that, via novel alignment algorithm, defines a constant factor approximation (in the worst case) to this global optimum.
We then provide algorithms that recover this witness set and yield the first provable constant factor approximation for the: (i) alignment problem in $O(n)$ expected time, and (ii) registration problem in polynomial time.
Such small witness sets exist for many variants including points in $d$-dimensional space, outlier-resistant cost functions, and different correspondence types.

Extensive experimental results on real and synthetic datasets show that our approximation constants are, in practice, close to $1$, and up to x$10$ times smaller than state-of-the-art algorithms.
\end{abstract}

\section{Introduction} \label{sec:intro}
Consider the set $P$ of known 3D landmarks mounted on a car, and the set $Q$ of the same 3D landmarks as currently observed via an external 3D camera, say, a few seconds later.
Suppose that we wish to compute the new car's position and orientation, relative to its starting point. These can be deduced by recovering the rigid transformation (rotation and translation) that align $P$ to $Q$. In this \emph{alignment problem}, we assume that the correspondence (matching) between every point in $P$ to $Q$ is known. When this matching is unknown, and needs to be computed, the  problem is known as the \emph{registration problem}. It is a fundamental problem in computer vision ~\cite{newcombe2011kinectfusion,makadia2006fully,salvi2007review,whelan2015real} with many applications in robotics~\cite{magnusson2007scan,pomerleau2015review,garcia20043d} and autonomous driving~\cite{zhang2015visual}.

\begin{table*}[t]
\caption{\textbf{Example contributions. }Variants of the problems~\eqref{eqPCAlignment}--\eqref{eqmain3} that we approximate in this paper, either using: (i) Corollary~\ref{corApproxCost} (known correspondence), or (ii) Theorem~\ref{theorem:matching} (unknown correspondence). Let $P=\br{p_1,\cdots,p_n}$ and $Q=\br{q_1,\cdots,q_n}$ be two sets of points in $\REAL^d$, let $z,r, T > 0$, and let $w = d^{\abs{\frac{1}{z}-\frac{1}{2}}}$. Formally, we wish to minimize $\cost(P,Q,(R,t)) = f\left( \lip \left(D\left(Rp_1-t,q_1\right)\right),\cdots, \lip\left(D\left(Rp_n-t,q_n\right)\right)\right)$ for functions $D:\REAL^d\times\REAL^d \to [0,\infty)$, $\lip:[0,\infty) \to [0,\infty)$ and $f:\REAL^n \to [0,\infty)$ as in Definition~\ref{def:cost}.
Rows marked with a $\star$ can also be approximated in linear time with high probability and bigger approximation factors, using Theorem~\ref{lemProbAlignment}.}
\begin{adjustbox}{width=\textwidth}
\small
\begin{tabular}{ | c | c | c | c | c | c | c |}
\hline
Use case & $f(v)$ & $\lip(x)$ & $D(p,q)$ & \makecell{Optimization Problem\\$\cost(P,Q,(R,t))$} & \makecell{Approximation\\Factor} & \makecell{Matching $\M$\\Given?}\\
\hline
\makecell{Sum of distances $\star$} & $\norm{v}_1$ & $x$ & $\norm{p-q}$ & \makecell{$\sum_{i=1}^n \norm{Rp_i-t-q_{\M(i)}}$} & $(1+\sqrt{2})^{d}$ & \makecell{Not\\necessary}\\
\hline
\makecell{Sum of squared distances $\star$} & $\norm{v}_1$ & $x^2$ & $\norm{p-q}$ & $\sum_{i=1}^n \norm{Rp_i-t-q_{\M(i)}}^2$ & $(1+\sqrt{2})^{2d}$ & \makecell{Not\\necessary}\\
\hline
\makecell{Sum of distances with\\noisy data using M-estimators} & $\norm{v}_1$ & \makecell{$\min\br{x,T}$} & $\norm{p-q}$ & \makecell{$\sum_{i=1}^n \min\br{\norm{Rp_i-t-q_{\M(i)}},T}$} & $(1+\sqrt{2})^{d}$ & \makecell{Not\\necessary}\\
\hline
\makecell{Sum of $\ell_z$ distances to the\\power of $r$ $\star$} & $\norm{v}_1$ & $x^r$ & $\norm{p-q}_z$ & $\sum_{i=1}^n \norm{Rp_i-t-q_{\M(i)}}_z^r$ & $w^r(1+\sqrt{2})^{dr}$ & \makecell{Not\\necessary}\\
\hline
\makecell{Sum of $\ell_z$ distances to the\\power of $r$ with\\$k\geq 1$ outliers} & \makecell{Sum of the $n-k$\\smallest entries\\of $v$} & $x^r$ & $\norm{p-q}_z$ & \makecell{$ \displaystyle\sum_{i\in \I \subset \br{1,\cdots, n}, |\I| = n-k}\norm{Rp_i-t-q_{\M(i)}}_z^r$ } & $w^r(1+\sqrt{2})^{dr}$ & \makecell{Yes}\\
\hline
\end{tabular}
\end{adjustbox}
\label{table:contrib}
\end{table*}

\begin{figure}
  \centering
  \includegraphics[width=0.35\textwidth]{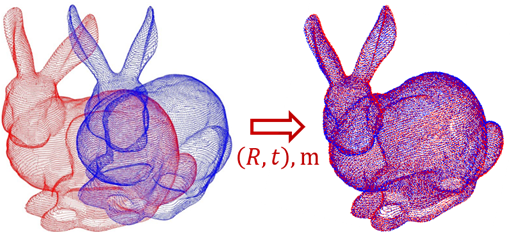}
  \caption{(left) The blue bunny is a translated and rotated version of the red bunny. (right) The registration problem is to recover this transformation, which includes matching each blue point to a corresponding red point.}\label{fig:bunny}
\end{figure}

\paragraph{Alignment. }In the alignment problem the input consists of two ordered sets $P = \br{p_1,\cdots,p_n}$ and $Q = \br{q_1,\cdots,q_n}$ in $\REAL^d$, where $d=3$ in the previous application, and the goal is to minimize
\begin{equation}\label{eqPCAlignment}
\sum_{i=1}^{n} D(Rp_i-t,q_{i}),
\end{equation}
over every \emph{alignment} (rigid transformation) $(R,t)$ consisting of a rotation matrix $R \in \REAL^{d\times d}$ (an orthogonal matrix whose determinant is $1$), and a translation vector $t\in \REAL^d$, and where $D(p,q) = \norm{p-q}$ is the Euclidean ($\ell_2$) distance between a pair of points $p,q \in \REAL^d$.
Here, the sum is over the distance between every point $p_i\in P$ to its corresponding point $q_i\in Q$.
This correspondence may be obtained using some auxiliary information, like point-wise descriptors e.g., SIFT~\cite{lowe1999object},
visual tracking of points~\cite{OptiTrack, Vicon}, or the use of predefined shapes and features~\cite{nasser2015coresets, rabinovich2020cobe}.

To our knowledge, the only provable approximation to the optimal \emph{global minimum} of~\eqref{eqPCAlignment} is for its variant where $D(p,q)$ is replaced by $\lip(D(p,q)) = \norm{p-q}^2$, i.e., \emph{squared} Euclidean distance. In this special case, the optimal solution is unique and easy to compute: $t$ is simply the vector connecting the two centers of mass of $P$ and $Q$, and $R\in\REAL^{d\times d}$ can be computed using Singular Value Decomposition~\cite{golub1971singular} as described in~\cite{kabsch1976solution}. This paper gives the first provable non-trivial approximation algorithm for~\eqref{eqPCAlignment}, while also handling an even wider range of functions.

\paragraph{Registration. }The registration problem does not assume the correspondence between $P$ and $Q$ is given, that is, we do not know which point in $Q$ matches $p_i \in P$.
Therefore, besides the rigid motion, the correspondence needs also to be extracted based solely on the two given point clouds, resulting in a much more complex problem with a large number of local minima; see Fig.~\ref{fig:bunny}. Formally, it aims to minimize
\begin{equation} \label{eqmain1}
\sum_{i=1}^{n}
\lip\left(D(Rp_i-t,q_{\M(i)})\right),
\end{equation}
over every alignment $(R,t)$ and correspondence function $\M:\{1,\ldots,n\}\to\{1,\ldots,n\}$; see recent survey~\cite{tam2012registration}. Here, a natural selection for $\lip$ is $\lip(x) = x^2$. The set $Q$ here is assumed to be of size $n$ for simplicity only, but can be of any different size.

Unlike~\eqref{eqPCAlignment}, we do not know a provable approximation to~\eqref{eqmain1}, even for $\ell(x) = x^2$.
The most commonly used solution for this problem, both in academy and industry, is the Iterative Closest Point (ICP) heuristic~\cite{besl1992method}; see Section~\ref{sec:relatedWork}.
In this paper we provide a ``provable ICP'' version which approximates the global optimum of this problem.

\paragraph{Alternative cost functions. }
When dealing with real-world data, noise and outliers are inevitable.
One may thus consider alternative cost functions, rather than the sum of squared distances (SSD) above, due to its sensitivity to such corrupted input.
A natural more general cost function would be to pick e.g., $\lip(x) = x^r$ for $r > 0$, which is more robust to noise when $r \in (0,1]$.
Alternatively, for handling outliers, a more suitable function would be $\lip(x) = \min\br{x,T}$ for some threshold $T>0$, or the common Tuckey or Huber losses, or any other robust statistics function~\cite{huber2011robust}.

To completely ignore these (unknown) faulty subsets of some paired data we may consider solving
\begin{equation} \label{eqmain3}
\min_{(R,t)} \sum_{i\in \I \subset \br{1,\cdots, n}, |\I| = n-k} \lip\left(D(Rp_i-t,q_{\M(i)})\right),
\end{equation}
where $k\leq n$ is the number of outliers to ignore.

In this paper we suggest a general framework for provably approximating the global minimum of the alignment and registration problems, including formulations~\eqref{eqPCAlignment}--\eqref{eqmain3}.

\subsection{Related Work} \label{sec:relatedWork}
The most common method for solving the registration problem in~\eqref{eqmain1}, for $\ell(x) = x^2$, is the ICP algorithm~\cite{chen1992object,besl1992method}.
The ICP is a local optimization technique, which alternates, until convergence, between solving the correspondence problem and the rigid alignment problem.
Over the years, many variants of the ICP algorithm have been suggested; see surveys in~\cite{pomerleau2013comparing} and references therein.
However, these methods usually converge to local and not global minimum if not initialized properly.

\textbf{Estimation maximization approaches. }
To overcome the ICP limitations, probabilistic methods~\cite{rangarajan1997robust} have been suggested, making use of GMMs, treating one point set as the GMM centroids, and the other as data points~\cite{joshi1995problem,wells1997statistical,cross1998graph,luo2003unified,mcneill2006probabilistic, hermans2011robust,biber2003normal}. This category also includes the widely used Coherent Point Drift (CPD) method~\cite{myronenko2010point}.

\textbf{Learning-based approaches. }
Learning dedicated features for this task was shown to enhance the output alignment~\cite{wang2019deep}.
In~\cite{aoki2019pointnetlk}, a deep learning model was
combined with a modified version of the known Lukas \& Kanade algorithm. Recently, 
an unsupervised deep learning based approach was proposed in~\cite{halimi2019unsupervised}.

\textbf{Alternative approaches. }
Some results make use of the Fourier domain~\cite{makadia2006fully}, and use correlation of kernel density estimates (KDE)~\cite{tsin2004correlation}, which scales poorly as the input size increases.
Other results use a Branch and Bound scheme to try and compute the global minimum~\cite{olsson2008branch, dym2019linearly,yang2015go}.

\textbf{Common limitations. }
The previously mentioned methods share similar properties and either (i) support only the simple sum of squared distances function or $d=3$, (ii) they converge to a local minima due to bad initialization, (iii) give optimality guarantees, if any, only on a sub-task of the registration pipeline, and lack such guarantees relative to the global optimum of the registration problem, (iv) their convergence time is impractical or depends on the data itself, 
or (v) require a lot of training data.
To our knowledge, no provable approximation algorithms have been suggested for tackling~\eqref{eqmain1}, even for $\lip(x) = x$.

\subsection{Our Contribution}
Our contributions in this paper are as follows.
\renewcommand{\labelenumi}{(\roman{enumi})}
\begin{enumerate}[leftmargin=0cm]
\item A proof that \emph{every} input pair of point clouds admits a witness set of $d$ points from $P$ and corresponding $d$ points from $Q$ that defines a constant factor approximation. Such a (usually different) witness set exists for all the versions of the alignment and registration problem, including problems~\eqref{eqPCAlignment}--\eqref{eqmain3}; see Section~\ref{sec:approxScheme}.

%
\item The first provable constant factor approximation algorithms for the above alignment and registration problems, including sum of distances to the power of $r>0$ and sum of M-estimators; see Corollary~\ref{corApproxCost}, Theorem~\ref{theorem:matching} and Table~\ref{table:contrib}. These algorithms run in polynomial time; see Algorithms~\ref{a107} and~\ref{AlgMatching}.

\item A probabilistic \emph{linear time} algorithm for solving the alignment problem, which also supports sum of distances to the power of $r>0$ and sum of M-estimators; see Algorithm~\ref{alg:prob} and Theorem~\ref{lemProbAlignment}.

\item Extensive experimental results on synthetic and real-world datasets which demonstrate the effectiveness and accuracy of both the polynomial time and randomized algorithms, as compared to state of the art methods; see Section~\ref{sec:ER}.
The results show that the approximation factor obtained in practice is much smaller than the theoretically predicted factor.
We provide full open-source code for our algorithms~\cite{opencode}.
\end{enumerate}

\subsection{
             Witness Set} \label{sec:approxScheme}

\begin{figure}
  \centering
  \includegraphics[width=0.4\textwidth]{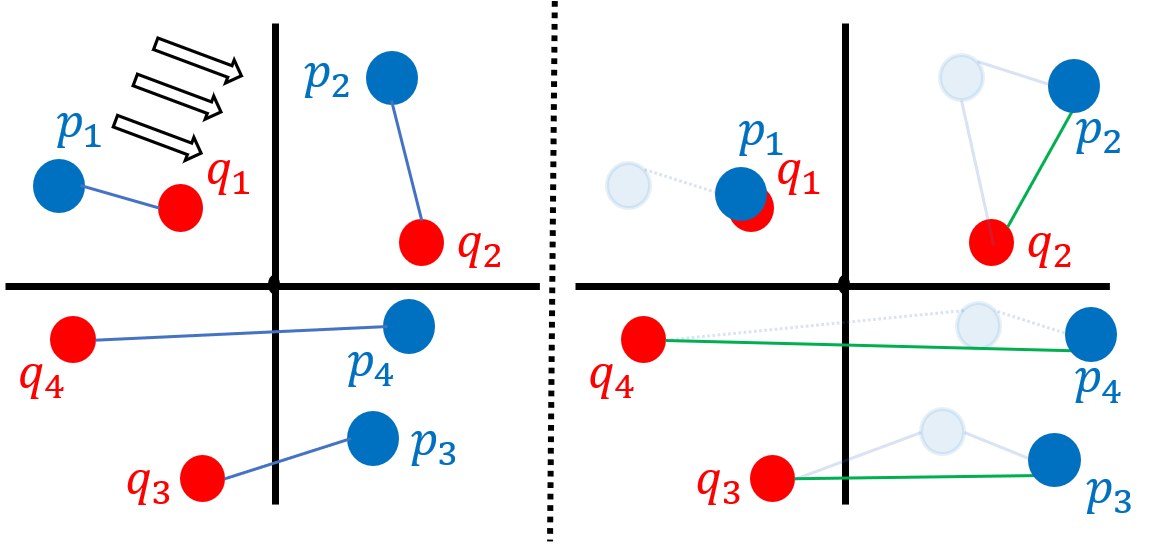}
  \caption{(Left): Two corresponding sets of points $P$ (in blue) and $Q$ (in red), where $p_1$ and $q_1$ have the smallest distance among all pairs.
  (Right): Translating $P$ by $t=p_1-q_1$ (i.e., $p_1$ now intersects $q_1$).
  By the triangle inequality, each distance $\norm{p_i-t-q_i}$ (green lines) is at most $2\cdot \norm{p_i-q_i}$ (blue lines).
  }\label{fig:approxTranslation}
\end{figure}

\begin{figure}
  \centering
  \includegraphics[width=0.4\textwidth]{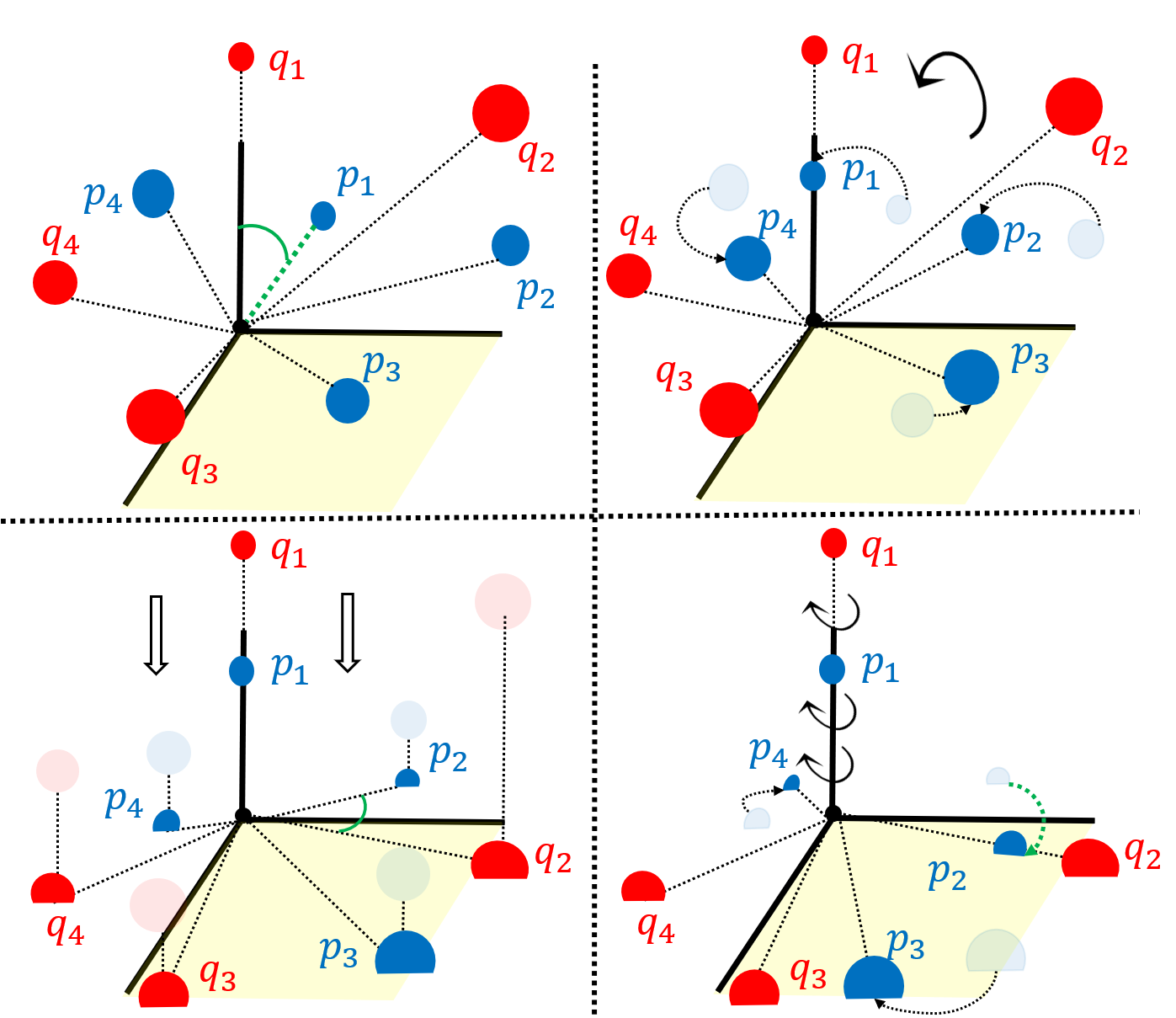}
  \caption{Illustration of Algorithm~\ref{a106}.
  (Top left): Two sets of corresponding points $P$ (in blue) and $Q$ (in red).
  (Top right): Rotating $P$ such that some $p\in P$ aligns with its corresponding $q \in Q$.
  (Bottom left): Projecting the rotated set $P$ and the set $Q$ onto the plane orthogonal to $q_1$. (Bottom right): Rotating the projected $P$ such that one of its points aligns with its corresponding point from $Q$.
  Observe that the initial aligned pair of points $(p_1,q_1)$ are not affected by the proceeding steps.}\label{fig:rotApprox}
\end{figure}

In this section we introduce our novel technique.
For simplicity we assume the correspondence between $P$ and $Q$ is given.
A generalization to the case with unknown correspondence is detailed in Section~\ref{sec:unknownMatching}.

Our main technical result is that for every corresponding ordered point sets $P = \br{p_1,\cdots,p_n},Q = \br{q_1,\cdots,q_n} \subseteq \REAL^d$, every cost function $\cost$ which satisfies some set of properties, e.g., sum of distances (see Definition~\ref{def:cost}), and every possible alignment $(R^*,t^*)$, there is a subset of $d$ points from $P$ and $d$ points from $Q$, which we call a \emph{witness set}.
Those small subsets determine, using our novel algorithm, an alignment $(R',t')$, that approximates the cost of $(R^*,t^*)$ in the following sense
\[
\cost(P,Q,(R',t')) \leq c \cdot \cost(P,Q,(R^*,t^*))
\]
for small constant $c>0$.
Here, $\cost$ assigns a non-negative value for every pair of input point sets and alignment.

For the existence proof's sake, we assume that the alignment $(R^*,t^*)$, which might be the optimal alignment between the two sets, is known beforehand. We drop assumption in practice.
The proof then applies a series of steps which alter this initial alignment $(R^*,t^*)$, until a different alignment $(R',t')$ is obtained, where a (witness) set of points from $P$ and $Q$ satisfies a sufficient number of known constraints, making it feasible (given this witness set) to recover $(R',t')$.
Each step in this series is guaranteed to approximate the cost of it's preceding step.
Hence, the cost of $(R',t')$ approximates the initial (optimal) cost of $(R^*,t^*)$.
The above scheme can also be applied to different problems from other fields; see e.g.,~\cite{jubran2019provable}.
The steps are as follows:
\renewcommand{\labelenumi}{(\roman{enumi})}
\begin{enumerate}[leftmargin=0cm]
  \item Consider the set $P'$ obtained by applying the optimal (unknown) alignment $(R^*,t^*)$ to $P$.
  Now, consider the single corresponding pair of points $p' = R^*p-t^* \in P'$ and $q \in Q$ which have the closest distance $\norm{p' - q}$ between them among all matched pairs.
  Using the triangle inequality, one can show that translating the set $P'$ by $p'-q$ (that is, such that $p'$ now intersects $q$) would not increase the pairwise distances of the other pairs of points by more than a multiplicative factor of $2$; see Fig.~\ref{fig:approxTranslation}.
  Hence, we proved the existence of some translation $t'$ of $P'$, where some $p \in P$ intersects its corresponding $q\in Q$, and where the cost is larger than the initial optimal cost by at most a factor of $2$. We thus narrowed down the space of candidate alignments. For the following step, without loss of generality, assume that $p'$ and $q$ are located at the origin. \label{bullet0}

  \item Similarly, we prove there is a corresponding pair of points $p' \in P'$ and $q \in Q$ such that aligning their direction vectors via a rotation matrix $R'$, i.e., $R'\frac{p'}{\norm{p'}} = \frac{q}{\norm{q}}$, would increase the pairwise distances of the other pairs by at most a small factor. $p'$ and $q$ are the pair with the smallest angle between them; see details in Lemma~\ref{lemRotApprox} in the appendix.
  \label{bullet1}

  \item We can repeat the step described in (ii) above recursively as follows: Find such a pair $(p',q)$, align their direction vectors, project the two sets of points onto the hyperplane orthogonal to the direction vector of $q$, and then continue recursively (at most $d-2$ times). Such a projection insures that the next uncovered rotation will maintain the alignment of $(p',q)$; see Fig.~\ref{fig:rotApprox}.
  Each such step proves \emph{the existence} of yet another corresponding pair of points which contribute at least $1$ constraint on $R'$, without damaging the cost by more than a constant factor. Hence, there exist $d-1$ pairs of points which uniquely determine our approximated rotation $R'$. Now, these (unknown) $d-1$ pairs along with the (unknown) pair from (i) above (which determine the translation $t'$), are called a witness set. Given the witness set, the approximated alignment $(R',t')$ can be recovered. To this end, Algorithm~\ref{a107} iterates over every set of $d$ corresponding pairs from $P$ and $Q$, and computes the alignment they determine. We then pick the alignment $(R',t')$ that minimizes $\cost(P,Q,(R',t'))$, which is guaranteed to approximate $\cost(P,Q,(R^*,t^*))$; see Theorem~\ref{lemApproxAlignment} and Corollary~\ref{corApproxCost}.
  \label{bullet2}
\end{enumerate}

The claims above imply that there is at least one such witness set. However, our experimental results suggest that many such $d$ pairs can serve as witness sets, in practice, and yield a good approximation result.

\section{Approximation for the Alignment Problem} \label{sec:givenMatching}
In this section, we suggest a novel approximation algorithm for the alignment problem; see Algorithm~\ref{a107}.
The algorithm computes the witness set described in Section~\ref{sec:approxScheme}. The formal statement is given in Theorem~\ref{lemApproxAlignment}, and its full proof is placed in the appendix.

\begin{figure}
  \centering
  \includegraphics[width=0.4\textwidth]{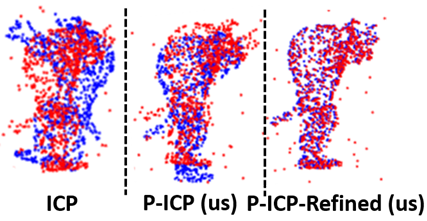}
  \caption{Registration visualization using the Armadillo model, $n=1000$ and $k=20\%$ outliers. (Left) $\ICP(P,Q)$, (middle) $\matchAlg(P,Q, \cost, \iters)$, (right) $\matchAlgICP(P,Q, \cost, \iters)$. $\cost$ is the SSD with $\lip$ as the threshold M-estimator and $\iters=3000$; see Section~\ref{ER:registration}.
} \label{fig:ArmadilloVis}
\end{figure}


\textbf{Notation.}
We denote $[n]=\br{1,\cdots,n}$ for any integer $n \geq 1$ and by $\vec{0}$ the origin of $\REAL^d$.
We assume every vector is a column vector. The $d$ dimensional identity matrix is denoted by $I_d \in \REAL^{d\times d}$.
Let $\rot(d)$ be the set of all rotation matrices in $\REAL^d$.
For $t\in\REAL^d$ and $R\in \rot(d)$, the pair $(R,t)$ is called an \emph{alignment}.
We define $\alignments(d)$ to be the union of all possible $d$-dimensional alignments.

In what follows, given some Linear subspace $X$ of $\REAL^d$, we define $\R_X$ to be the set of all rotation matrices $R$ such that $p \in X$ if and only if $Rp\in X$; see Fig.~\ref{fig:Rx}.
\begin{definition} \label{defRq}
Let $\sdim\in\br{0,\cdots,d}$, let $X$ be a $\sdim$-dimensional subspace of $\REAL^d$, and let $V_X \in \REAL^{d\times d}$ be a unitary arbitrary matrix whose $\sdim$ leftmost columns span $X$.
We define
\[
\R_X = \br{V_X\left( \begin{array}{ccc}
		R & \mathbf{0} \\
		\mathbf{0} & I_{d-\sdim}\end{array} \right)V_X^T \mid R \in \rot(\sdim)},
\]
for $\sdim \geq 2$, and $\R_X = \br{I_d}$ otherwise.
\end{definition}
For example, if $X$ is spanned by the first $\sdim$ axis of $\REAL^d$ then $\R_X$ contains all the rotation matrices in $\rot(d)$ which, when multiplied by any point in $p\in \REAL^d$, affect only the first $\sdim$ coordinates of $p$.
If not, $V_X^T$ aligns $X$ with the first $\sdim$ axis of $\REAL^d$ and $V_X$ does the inverse rotation; see Fig.~\ref{fig:Rx}.

\begin{figure}
  \begin{minipage}[c]{0.15\textwidth}
    \includegraphics[scale=0.33]{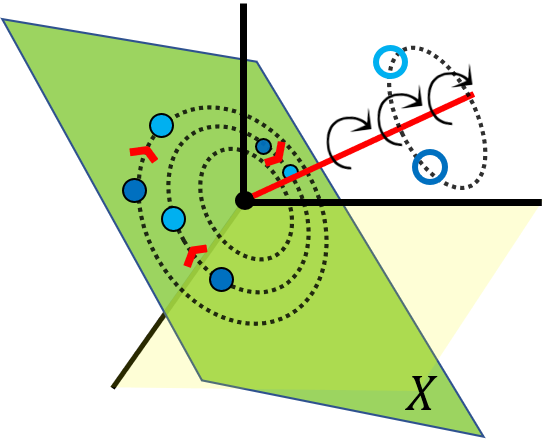}
  \end{minipage}\hfill
  \begin{minipage}[c]{0.31\textwidth}
    \caption{A plane $X$ in $\REAL^3$. $\R_X$ contains all the rotation matrices that map points $p\in X$ (solid dark blue) to other points $p' \in X$ (solid light blue), and points $q \not\in X$ (dark blue circles) to points $q' \not\in X$ (light blue circles).}\label{fig:Rx}
  \end{minipage}
\end{figure}

\textbf{Overview of Algorithm~\ref{a106}.}
Assuming the translation is given, Algorithm~\ref{a106} recovers the rotation determined by some subset of pairs, as described in Section~\ref{sec:approxScheme}.
It gets as input two sets $\br{p_1,\cdots,p_{\jdim-1}}$ and $\br{q_1,\cdots,q_{\jdim-1}}$ of points that are contained in a $\jdim$-dimensional subspace $\pi$ of $\REAL^d$.
In Line~\ref{line:compR}, the algorithm computes a rotation matrix $R\in \R_\pi$ that aligns the directions of $p_1$ and $q_1$. In Line~\ref{linea12} the algorithm computes an orthogonal matrix $W$ whose column space spans the $(\jdim-1)$-dimensional subspace orthogonal to $q_1$. Then, in Line~\ref{lineRecursieRot}, it projects the rotated points $Rp_2,\cdots,Rp_{\jdim-1}$ and the points $q_2,\cdots,q_{\jdim-1}$ onto $\pi'$, and continues recursively. Hence, the recursive call will compute a rotation matrix $R'\in \R_{\pi'}$, which aligns the directions of $WW^TRp_2$ and $WW^Tq_2$, but also guarantees that $RR'$ still aligns $p_1$ and $q_1$. In other words, the rotation computed at the $i$'th recursive call, when multiplied by the previous rotations, will not affect the alignment of the vectors from the previous $i-1$ iterations; see illustration in Fig.~\ref{fig:rotApprox}.
The algorithm terminates after $\jdim-2$ such recursive calls.

\begin{algorithm}
	\caption{ \textsc{\algnamegetrot$(P,Q)$}} \label{a106}
	\SetKwInOut{Input}{Input}
	\SetKwInOut{Output}{Output}
	\Input{Two sets of points $P = \br{p_1,\cdots,p_{\jdim-1}}$ and $Q = \br{q_1,\cdots,q_{\jdim-1}}$ in a $\jdim$-dimensional subspace $\pi$ of $\REAL^d$, such that $\norm{p_1} , \norm{q_1} > 0$.}
	\Output{ A rotation matrix $R \in \R_{\pi}$.}
	
	$R :=$ an arbitrary rotation matrix in $\R_\pi$ that satisfies $\frac{Rp_1}{\norm{p_1}}= \frac{q_1}{\norm{q_1}}$. \label{line:compR}\\
	\If{$\jdim = 2$}{$ \Return\: R$}
	
	$W :=$ an arbitrary matrix in $\REAL^{d\times (d-1)}$ whose columns are mutually orthogonal unit vectors, and its column space spans the hyperplane orthogonal to $q_1$, i.e., $\big[W \mid \frac{q_1}{\norm{q_1}}\big] \in \REAL^{d\times d}$ forms a basis of $\REAL^d$. \label{linea12}\\

$p_i' := WW^TRp_i$ and $q_i' := WW^Tq_i, \forall i\in [\jdim-1]$

$S := \algnamegetrot(\br{p_2', \cdots,p_{\jdim-1}'}, \br{q_2', \cdots, q_{\jdim-1}'})$. \label{lineRecursieRot}

	$ \Return\: S R$
\end{algorithm}

\textbf{Overview of Algorithm~\ref{a107}.} Algorithm~\ref{a107} simply implements the exhaustive (subset) search proposed in Bullet (iii) of Section~\ref{sec:approxScheme}. Observe that this algorithm is ``embarrassingly parallel''~\cite{foster1995designing}.
Furthermore, as empirically demonstrated in Section~\ref{sec:ER}, it suffices to iterate over only a fraction of those subsets to obtain a good approximation.

\begin{algorithm}
	\caption{ {\algname$(P,Q)$} }\label{a107}
	\SetKwInOut{Input}{Input}
	\SetKwInOut{Output}{Output}
	\Input{A pair of sets $P=\br{p_1,\cdots,p_n}$ and $Q=\br{q_1,\cdots,q_n}$ in $\REAL^d$.}
	\Output{A set of alignments; see Theorem~\ref{lemApproxAlignment}}
		$ M := \emptyset$.\\
		\For {every $i \in [n]$ \label{lineESalignment}}
		{		
		    $p_j' := p_j - p_i$ and $q_j' := q_j-q_i$, $\forall j\in [n]$.
		
		    $A_i := \emptyset$
		
		    \For {every distinct $i_1, \cdots, i_{d-1}\in [n] \setminus \br{i}$
		  \label{line:algrotfor}}
        		{		
        			$A_i := A_i \cup \algnamegetrot\left(\br{p_{i_1}',\cdots ,p_{i_{d-1}}'},\br{q_{i_1}',\cdots ,q_{i_{d-1}}'}\right)$ \label{lineGetRot} \tcc{see Alg.~\ref{a106}}
        		}
		
			$M := M \cup \br{(R,R p_{i}-q_{i})\mid R\in A_i }$ \label{linecompt}\\
		}
	$\Return\: M$
\end{algorithm}


\begin{theorem} \label{lemApproxAlignment}
Let $P = \br{p_1,\cdots,p_n}$ and $Q = \br{q_1,\cdots,q_n}$ be two ordered sets each of $n$ points in $\REAL^d$.
Let $M \subseteq \alignments(d)$ be the output of a call to \algname$(P,Q)$; see Algorithm~\ref{a107}.
Then, for every $(R^*,t^*) \in \alignments(d)$, there exists $(\hat{R},\hat{t}) \in M$ that satisfies the following for every $i\in [n]$,
\[
\norm{\hat{R}p_i-\hat{t}-q_i} \leq (1+\sqrt{2})^{d} \cdot \norm{R^*p_i-t^*-q_i}.
\]
Furthermore, $(\hat{R},\hat{t})$ is computed in $n^{O(d)}$ time.
\end{theorem}
\ifproofs
\begin{proof}
Put $(R^*,t^*)\in \alignments(d)$. Without loss of generality assume that $R^*$ is the identity matrix of $\REAL^d$ and that $t^* = \vec{0}$, otherwise rotate and translate the coordinates system.

For every $i\in [n]$ let $t_i = p_i-q_i$ and let $k\in \argmin_{i\in [n]} \norm{t_i}$. Put $i\in [n]$.
Since $\norm{t_k} \leq \norm{t_i} = \norm{p_i-q_i}$, it holds that
\begin{equation} \label{eqTransApprox}
\norm{p_i-t_k-q_i} \leq \norm{p_i-t_k-p_i} + \norm{p_i-q_i} = \norm{t_k} + \norm{p_i-q_i} \leq  2\cdot \norm{p_i-q_i}.
\end{equation}

Let $p_i' = p_i-p_k$, $q_i' = q_i-q_k$, $P' = \br{p_1',\cdots,p_n'}$ and $Q' = \br{q_1',\cdots,q_n'}$.
Consider the set $A_i$ computed at Line~\ref{lineGetRot}.
By Lemma~\ref{lemApproxD} there is $\hat{R} \in A_i$ such that
\begin{equation} \label{prop1}
\norm{\hat{R}p_i'-q_i'} = \norm{\hat{R}p_i'-q_i'} \leq (1+\sqrt{2})^{d-1} \cdot \norm{\idmat(d)p_i'-q_i'} = (1+\sqrt{2})^{d-1} \cdot \norm{p_i'-q_i'}.
\end{equation}

Let $\hat{t} = t_k + \hat{R}p_k -p_k$. We now obtain that
\begin{align}
\norm{\hat{R}p_i-\hat{t}-q_i} & = \norm{\hat{R}p_i-t_k-\hat{R}p_k+p_k-q_i} \label{eq11}\\
& = \norm{\hat{R}(p_i-p_k)-(p_k-q_k)+p_k-q_i} \nonumber\\
& = \norm{\hat{R}(p_i-p_k)+q_k-q_i} \nonumber\\
& = \norm{\hat{R}p_i'-q_i'} \label{eq13}\\
& \leq (1+\sqrt{2})^{d-1} \cdot \norm{p_i'-q_i'} \label{eq14}\\
& = (1+\sqrt{2})^{d-1} \cdot \norm{p_i-p_k-q_i+q_k} \nonumber\\
& = (1+\sqrt{2})^{d-1} \cdot \norm{p_i-p_k + q_k-q_i} \nonumber\\
& = (1+\sqrt{2})^{d-1} \cdot \norm{p_i-t_k-q_i} \nonumber\\
& \leq 2(1+\sqrt{2})^{d-1} \cdot \norm{p_i-q_i} \label{eq16},
\end{align}
where~\eqref{eq11} holds by the definition of $\hat{t}$, \eqref{eq13} holds by the definition of $p_i'$ and $q_i'$, \eqref{eq14} holds by~\eqref{prop1}, and~\eqref{eq16} holds by~\eqref{eqTransApprox}.

Hence, it holds that
\begin{equation} \label{eqres}
\norm{\hat{R}p_i-\hat{t}-q_i} \leq 2(1+\sqrt{2})^{d-1}\cdot \norm{p_i-q_i} \leq (1+\sqrt{2})^{d}\cdot \norm{p_i-q_i}.
\end{equation}

In Line~\ref{lineESalignment} of Algorithm~\ref{a107} we go over every index $i\in [n]$. In Line~\ref{lineGetRot} we compute a set of rotation matrices $A_i$ that align the sets $P'=\br{p_1-p_i,\cdots,p_n-p_i} \setminus \br{\vec{0}},Q'=\br{q_1-q_i,\cdots,q_n-q_i} \setminus \br{\vec{0}}$ by a call to Algorithm~\ref{a106}, and in Line~\ref{linecompt} we compute a set of pairs that consist of a rotation matrix $R \in R_i$ and its corresponding translation vector $Rp_i-q_i$.
When $i=k$, we are guaranteed to compute the desired pair $(\hat{R},\hat{t})$ and add it to the set $M$.
Therefore, Theorem~\ref{lemApproxAlignment} holds by combining~\eqref{eqres} and $(\hat{R},\hat{t})\in M$.

The running time of Algorithm~\ref{a107} is $n^{O(d)}$ since we make $O(n^{r})$ calls to Algorithm~\ref{a106} at Line~\ref{lineGetRot}, each call takes $O(d^2)$ time.
\end{proof}
\fi

\subsection{Generalization} \label{sec:generalization}
In this section, we use a definition of a wide family of cost functions, and prove that the output $M$ of Algorithm~\ref{a107} contains alignments that approximate many functions from this family.
Note that each cost function may be optimized by a different candidate alignment in $M$.
See Table~\ref{table:contrib} for examples. In what follows, for $r>0$, an $r$-log-Lipschitz function is a function whose derivative may be large but cannot increase too rapidly (in a rate that depends on $r$).
For a higher dimensional function, we demand the previous constraint over every dimension individually. The formal definition is taken from~\cite{feldman2012data}, and is generalized in Definition~\ref{def:lip} at the appendix.

\begin{definition} [Cost function]\label{def:cost}
Let $P = \br{p_1,\cdots,p_n}$ and $Q = \br{q_1,\cdots,q_n}$ be two sets of $n$ points in $\REAL^d$.
Let $D:\REAL^d\times\REAL^d \to [0,\infty)$ be a function that assigns a non-negative weight for each pair of points in $\REAL^d$, $\lip:[0,\infty)\to [0,\infty)$ be an $r$-log-Lipschitz function and $f:[0,\infty)^n\to [0,\infty)$ be an $s$-log-Lipschitz function.
Let $(R,t)\in \alignments(d)$ be an alignment. We define
\[
\begin{split}
& \cost(P,Q,(R,t))\\
&\quad =f\left( \lip \left(D\left(Rp_1-t,q_1\right)\right),\cdots, \lip\left(D\left(Rp_n-t,q_n\right)\right)\right).
\end{split}
\]
\end{definition}

Observation 5 in~\cite{jubran2020aligning} (restated as Observation~\ref{obs:distToCost} in the appendix) states that in order to approximate any cost function from Definition~\ref{def:cost}, relative to some alignment $(R^*,t^*)$, it is sufficient to approximate the basic distance $\norm{R^*p_i-t^*-q_i}$ for every $i\in [n]$.
Combining the above with Theorem~\ref{lemApproxAlignment}, which provides such an approximation, yields the following claim; see full proof at the appendix.
\begin{corollary} \label{corApproxCost}
Let $P = \br{p_1,\cdots,p_n}$, $Q = \br{q_1,\cdots,q_n}$ be two ordered sets of $n$ points in $\REAL^d$, $z > 0$, and $w = d^{\abs{\frac{1}{z}-\frac{1}{2}}}$.
Let $\cost$, $r$ and $s$ be as defined in Definition~\ref{def:cost} for $D(p,q) = \norm{p-q}_z$.
Let $M \subseteq \alignments(d)$ be the output of a call to \algname$(P,Q)$; See Algorithm~\ref{a107}.
Then, there is $(R',t') \in M$ that satisfies,
\[
\cost(P,Q,(R',t')) \leq w^{rs}\cdot (1+\sqrt{2})^{drs} \cdot \min_{(R,t)}\cost(P,Q,(R,t)),
\]
where the minimum is over every $(R,t) \in \alignments$.
Moreover, $(R',t')$ is computed in $n^{O(d)}$ time.
\end{corollary}


\section{Approximation for the Registration Problem} \label{sec:unknownMatching}
Given two unmatched sets $P,Q \subseteq \REAL^d$ of $n$ points each and a cost function, our goal here is to solve the registration problem in~\eqref{eqmain1}, in polynomial time.

\textbf{Witness set for the registration problem. }Section~\ref{sec:approxScheme}
proves the existence of a witness set of size $d$ for every alignment of $P$ and $Q$.
Therefore,  when the matching is given as in Section~\ref{sec:givenMatching}, we exhaustive search over all possible subsets of $P$ of size $d$ to uncover the desired subset.
If the correspondence is not known, we can simply iterate over every possible set of $d$ points from $P$ and also over every possible set of $d$ points from $Q$, and uncover the alignment this set determines.
After finding the approximated alignment, we can transform $P$ using this alignment, and then compute the optimal matching between $Q$ and the transformed $P$ with respect to the given cost function (via nearest neighbor).
Thus, obtaining the desired approximated alignment and matching by decoupling the two problems and solving each of them individually. Algorithm~\ref{AlgMatching} is an implementation of the above scheme. Observe that this algorithm can be trivially parallelized.

\begin{algorithm}
	\caption{ {\algnamematching$(P,Q,\cost)$} }\label{AlgMatching}
	\SetKwInOut{Input}{Input}
	\SetKwInOut{Output}{Output}
	\Input{Sets $P=\br{p_1,\cdots,p_n} $, $Q=\br{q_1,\cdots,q_n}$ in $\REAL^d$ and a cost function.}
	\Output{An alignment and a matching function; see Theorem~\ref{theorem:matching}}
		$M := \emptyset$.
		
		\For {every $i_1, \cdots, i_{d},j_1,\cdots,j_d \in [n]$
		\label{lineESmatching}}
		{		
			$P' := \br{p_{i1},\cdots,p_{id}}$. \label{lineDefP2} and $Q':= \br{q_{j1},\cdots,q_{jd}}$. \label{lineDefQ2}
			
			$M := M \cup \algname(P',Q')$.
		}
		$S := \br{\left(R,t,\NN(P,Q,(R,t))\right) \mid (R,t)\in M}$. \label{lineMappingFunc}
		\tcc{$\NN(P,Q,(R,t))$ is the nearest neighbour matching between $Q$, and $P$ after applying $(R,t)$.}

 $(\tilde{R},\tilde{t},\tilde{\M}) \in \displaystyle\argmin_{(R',t',\M')\in S} \cost\left(P_{[\M']},Q,(R',t')\right)$. \label{lineMinS}

$\Return\:(\tilde{R},\tilde{t},\tilde{\M})$
\end{algorithm}

\begin{theorem} \label{theorem:matching}
Let $P = \br{p_1,\cdots,p_n}$, $Q = \br{q_1,\cdots,q_n}$ be two ordered sets of $n$ points in $\REAL^d$, $z >0$, and $w = d^{\abs{\frac{1}{z}-\frac{1}{2}}}$.
Let $\cost$ and $r$ be as in Definition~\ref{def:cost} for $D = \norm{p-q}_z$ and $f(v) = \norm{v}_1$. Let $(\tilde{R},\tilde{t},\tilde{\M})$ be the output of a call to \algnamematching$(P,Q,\cost)$; See Algorithm~\ref{AlgMatching}. Then,
\begin{equation}
\begin{split}
& \cost\left(P_{[\tilde{\M}]},Q,(\tilde{R},\tilde{t})\right)\\
& \quad\leq w^r (1+\sqrt{2})^{dr} \cdot \min_{(R,t,\M)}\cost\left(P_{[\M]},Q,(R,t)\right),
\end{split}
\end{equation}
where the minimum is over every alignment $(R,t)$ and permutation $\M$.
Moreover, $(\tilde{R},\tilde{t},\tilde{\M})$ is computed in $n^{O(d)}$ time.
\end{theorem}

Substituting $z=r=2$ in Theorem~\ref{theorem:matching} yields a provable approximated version of ICP. The proof is placed at the appendix. Furthermore, as our experiments in Section~\ref{sec:ER} show, the worst-case theory in Theorem~\ref{theorem:matching} is very pessimistic. In practice, the approximation constant is much smaller, and can be obtained much faster than the $n^{O(d)}$ suggested time.

\section{Run Time Improvement}
In this section, we propose a randomized algorithm that resembles Algorithm~\ref{a106}, which helps boost the running time of Algorithm~\ref{a107}. The full proofs are placed in the appendix.

\textbf{Overview of Algorithms~\ref{alg:probrot} and~\ref{alg:prob}. }
Unlike Algorithm~\ref{a106}, which simply aligns \emph{given} subsets of $P$ and $Q$, Algorithm~\ref{alg:probrot} takes as input two full point clouds, and identifies, by itself, a potential witness set.
Intuitively, points in $P$ with larger norm negatively affect our cost function more than points of smaller norm, when misaligned properly; see Fig.~\ref{fig:normSampling} at the appendix. Algorithm~\ref{alg:probrot} thus samples a pair of corresponding points $(p,q)$ with probability that depends on the norm of $p$ and rotates $P$ as to align the direction vectors of $p$ and $q$. Then, similarly to Algorithm~\ref{a106}, it projects the sets onto the hyperplane orthogonal to $q$ and continues recursively. 

Algorithm~\ref{alg:prob} applies Algorithms~\ref{alg:probrot}, with simple pre and post-processing, for only a small number of iterations.
Theorems~\ref{lemProbRotation} and~\ref{lemProbAlignment} prove that, with constant probability, the outputs of Algorithms~\ref{alg:probrot} and~\ref{alg:prob} approximate the desired cost.

\begin{algorithm}
	\caption{ {\algNameProbRot$(P,Q,r)$} }\label{alg:probrot}
	\SetKwInOut{Input}{Input}
	\SetKwInOut{Output}{Output}
	\Input{A pair of sets $P=\br{p_1,\cdots,p_n}$ and $Q=\br{q_1,\cdots,q_n}$ contained in a $\sdim$-dimensional subspace $\pi$ of $\REAL^d$, $r>0$.}
	\Output{A rotation matrix; see Lemma~\ref{lemProbRotation}}
	
	$w_i := \frac{\norm{p_i}^r}{\sum_{j \in [n]}\norm{p_j}^r}$ for every $i \in [n]$.
	
	Randomly sample an index $j \in [n]$, where $j = i$ with probability $w_i$.
	
	$R :=$ an arbitrary rotation matrix in $\R_\pi$ that satisfies $\frac{Rp_j}{\norm{p_j}}= \frac{q_j}{\norm{q_j}}$. \label{defineArbitraryR}
	
	\If{$\sdim = 2$}{$ \Return\: R$}
	
	$W := $ a matrix in $\REAL^{d\times (d-1)}$ such that $\big[W \mid \frac{q_j}{\norm{q_j}}\big] \in \REAL^{d\times d}$ forms a basis of $\REAL^d$. \label{lineCompW}
	
	$P' := \br{WW^TRp \mid p \in P \setminus \br{p_j}}$ and $Q' := \br{WW^Tq \mid q \in Q \setminus \br{q_j}}$
	
    $S := \algNameProbRot\left(P', Q'\right)$.

	$\Return\: SR$
\end{algorithm}
\begin{lemma} \label{lemProbRotation}
Let $P = \br{p_1,\cdots,p_n}$ and $Q = \br{q_1,\cdots,q_n}$ be two ordered sets of points in $\REAL^d$ and let $z>0$. 
Let $\cost$ be as defined in Definition~\ref{def:cost} for $f=\norm{v}_1$, some $r$-log Lipschitz $\lip$ and $D(p,q) = \norm{p-q}_z$.
Let $R'$ be an output of a call to \algNameProbRot$(P,Q,r)$; see Algorithm~\ref{alg:probrot}. Then, with probability at least $\frac{1}{2^{d-1}}$,
\[
\cost(P,Q,(R',\vec{0})) \leq \sigma \cdot \min_{R \in \rot(d)} \cost(P,Q,(R,\vec{0})),
\]
for a constant $\sigma$ that depends on $d$ and $r$. Furthermore, $R'$ is computed in $O(nd^2)$ time.
\end{lemma}

\begin{algorithm}
	\caption{ {\algNameApproxAlignment$(P,Q,r)$} }\label{alg:prob}
	\SetKwInOut{Input}{Input}
	\SetKwInOut{Output}{Output}
	\Input{Ordered sets $P=\br{p_1,\cdots,p_n}$ and $Q=\br{q_1,\cdots,q_n}$ in $\REAL^d$ and $r>0$.}
	\Output{A set of alignments; see Theorem~\ref{lemProbAlignment}}
	
	$M := \emptyset$.
	
	\For {$\frac{1}{\log\left(\frac{2^d}{2^d-1}\right)}$ iterations}
	{	
        Sample an index $j \in [n]$ uniformly at random

        $P' := \br{p-p_j \mid p \in P} \setminus \br{\vec{0}}$ and $Q' := \br{q-q_j \mid q \in Q}\setminus \br{\vec{0}}$

        $R' :=$ \algNameProbRot$(P',Q', r)$

        $M := M \cup \br{(R',R' p_j-q_j)}$
    }
	$\Return\: M$
\end{algorithm}
\begin{theorem} \label{lemProbAlignment}
Let $P = \br{p_1,\cdots,p_n}$ and $Q = \br{q_1,\cdots,q_n}$ be two ordered sets of points in $\REAL^d$ and let $z>0$.
Let $\cost$ be as defined in Definition~\ref{def:cost} for $f=\norm{v}_1$, some $r$-log Lipschitz $\lip$ and $D(p,q) = \norm{p-q}_z$.
Let $M$ be an output of a call to \algNameApproxAlignment$(P,Q,r)$; see Algorithm~\ref{alg:prob}. Then, with constant probability greater than $1/2$, there is an alignment $(R',t') \in M$ that satisfies
\[
\cost(P,Q,(R',t')) \leq \sigma\cdot \min_{(R,t) \in \alignments} \cost(P,Q,(R,t)),
\]
for a constant $\sigma$ that depends on $d$ and $r$. Furthermore, $M$ is computed in $O\left(\frac{nd^2}{\log\left(\frac{2^d}{2^d-1}\right)}\right)$ time.
\end{theorem}

\section{Experimental Results} \label{sec:ER}
In this section, we apply our approximation algorithms to solve either the alignment or the registration problems with various different cost functions. 

\textbf{Datasets. }We used the Bunny, Armadillo, and Asian Dragon models from the Stanford 3D scanning repository~\cite{curless1996volumetric, krishnamurthy1996fitting, turk1994zippered}. Those models were scaled to fit in the cube $[-0.5,0.5]^3$ due to the constraints of some competing methods (e.g., $\GOICP$). We also used a synthetic dataset comprising uniformly sampled $d$-dimensional points in $[-0.5,0.5]^d$.

\textbf{Generating $P$ and $Q$. }In all experiments, given some data model (real or synthetic), we uniformly sample $n$ points named $Q$. An alignment $(R,t)$ is generated, where $t$ is uniformly sampled such that $\norm{t} \leq 0.1$ and $R$ rotates the data around each axis by an angle uniformly sampled from $[-\pi,\pi]$. $P$ is then obtained by applying $(R,t)$ to $Q$ and adding Gaussian noise with zero mean and $\sigma^2$ variance.
In Section~\ref{ER:registration} we also apply a random shuffle to $P$.

\textbf{Evaluation. }The cost evaluation of an algorithm is done by plugging its output alignment $(R,t)$, along with its input $P$ and $Q$ into the cost function at hand (e.g.,~\eqref{eqPCAlignment} or~\eqref{eqmain3}). If not given, the optimal correspondence is trivially computed, after applying $(R,t)$ to $P$, via nearest neighbor.

\textbf{Setup. } All experiments were executed on the AWS platform, on a c$5$a.$8$xlarge machine with $32$ CPUs and $64$ Gigabytes of RAM.

\newcommand{\kabsch}{\texttt{Optimal-SSD}}
\newcommand{\alignAlg}{\texttt{Approx-Align}}
\newcommand{\alignAlgProb}{\texttt{Prob-Align}}
\newcommand{\alignRansac}{\texttt{RANSAC}}
\newcommand{\alignAdaptiveRansac}{\texttt{Adaptive-RANSAC}}

\subsection{Alignment Experiments} \label{ER:alignment}
In this test we assume the correspondences function is given. The goal is to demonstrate our algorithm's very fast recovery of a very good solution; see results in Fig.~\ref{fig:givenMatchingApprox}. We therefore chose the sum of squared distances cost function, for which we compute an optimal ground truth solution via SVD as explained in Section~\ref{sec:relatedWork}.
We tested $\alignAlg(P,Q, \cost, \iters)$- a combination of Algorithms~\ref{a107} and~\ref{alg:prob} as follows. Rather than iterating over the whole $n^{O(d)}$ possible subsets in Algorithm~\ref{a107}, we uniformly sample and test only $\iters \ll n^{O(d)}$ subsets. We then output the alignment that minimizes the given $\cost$.

\textbf{Discussion. }This experiment bridges the gap between theory and practice. The pessimistic theory in Corollary~\ref{corApproxCost} and Theorem~\ref{theorem:matching} suggested running $n^{O(d)}$ (millions) of iterations to obtain an approximation. Fig.~\ref{fig:givenMatchingApprox} shows that it suffices, in practice, to sample roughly $40$ subsets (in $d=3$) until an error of at most x$1.5$ the minimal cost is obtained.

\begin{figure}
  \centering
  \includegraphics[width=0.49\textwidth]{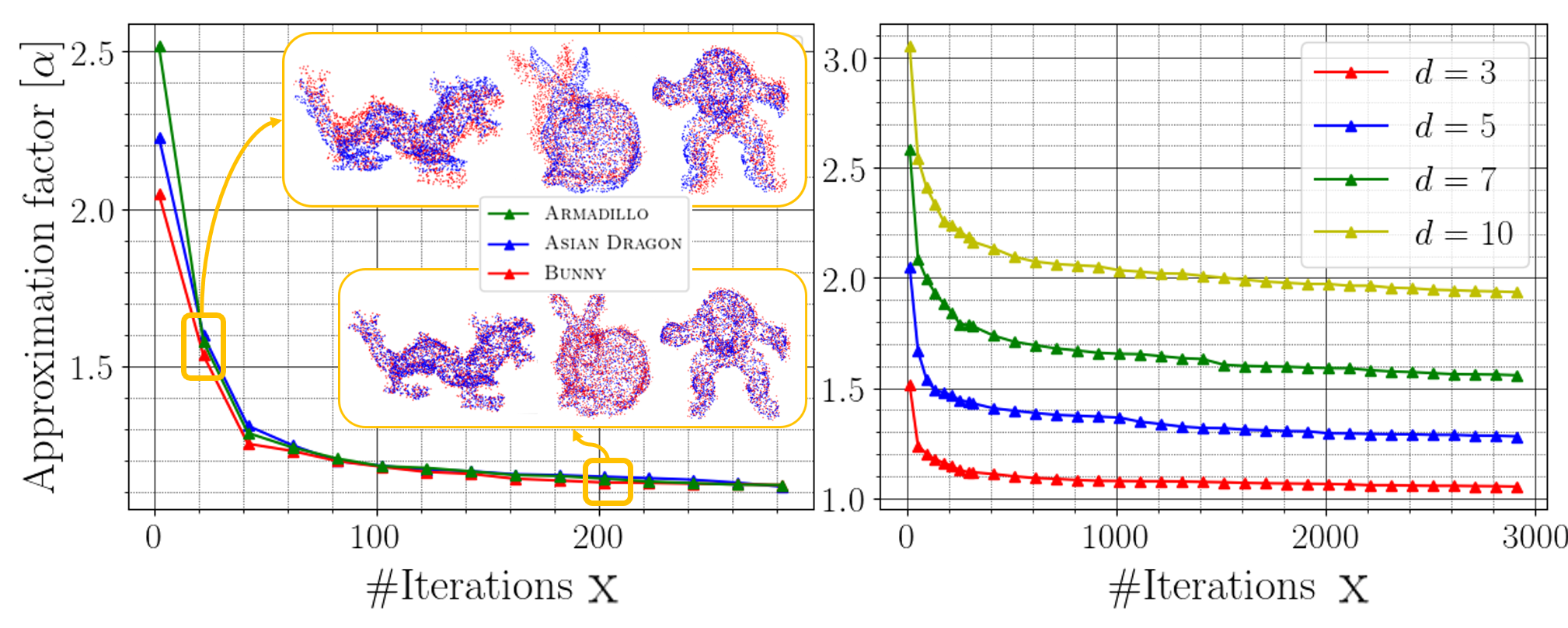}
  \caption{The approximation quality $\alpha$ as a function of the number of subsets $\iters$ tested via $\alignAlg(P,Q,\cost,\iters)$. $\alpha$ equals the final cost obtained via $\alignAlg$ divided by the ground truth optimal solution. (Left): The models Bunny, Armadillo, and Asian Dragon are used. The obtained alignment is also visualized. (Right): Synthetic data is used. In both figures $n=2500$ and $\sigma^2=0.01$.} \label{fig:givenMatchingApprox}
\end{figure}

\subsection{Registration Experiments} \label{ER:registration}

In this section we compared the following algorithms:
\textbf{(i)} $\ICP(P,Q)$ - An implementation of the common ICP algorithm~\cite{chen1992object}.
\textbf{(ii)} $\matchAlg(P,Q, \cost, \iters)$ - Provable ICP; An implementation in Python of Algorithm~\ref{AlgMatching}, which iterates over only $\iters$ sets of indices rather than the whole $n^{O(d)}$ iterations required.
\textbf{(iii)} $\matchAlgICP(P,Q, \cost, \iters)$ - Applying the output alignment of $\matchAlg$ to the set $P$, then refining the alignment by applying $\ICP$, only once, on $Q$ and the transformed $P$. Since Algorithms $\matchAlg$ and $\matchAlgICP$ are ``embarrassingly parallel'', they were parallelized over $32$ CPUs.
\textbf{(iv)} $\CPD(P,Q)$ - The popular Coherent Point Drift algorithm~\cite{myronenko2010point}.
\textbf{(v)} $\GOICP(P,Q)$ - The common ICP variant~\cite{yang2015go}.

We tested the accuracy of the above algorithms with multiple different cost functions. All tests in this section were conducted $10$ times and averaged. The variance is also presented in the graphs.
Fig.~\ref{fig:registrationMain} and Fig.~\ref{fig:registrationOutliers} present the results for noisy input data, and data containing outliers respectively. Visual comparison of our algorithms' is shown in Fig,~\ref{fig:ArmadilloVis}.

\begin{figure}
  \centering
  \includegraphics[width=0.49\textwidth]{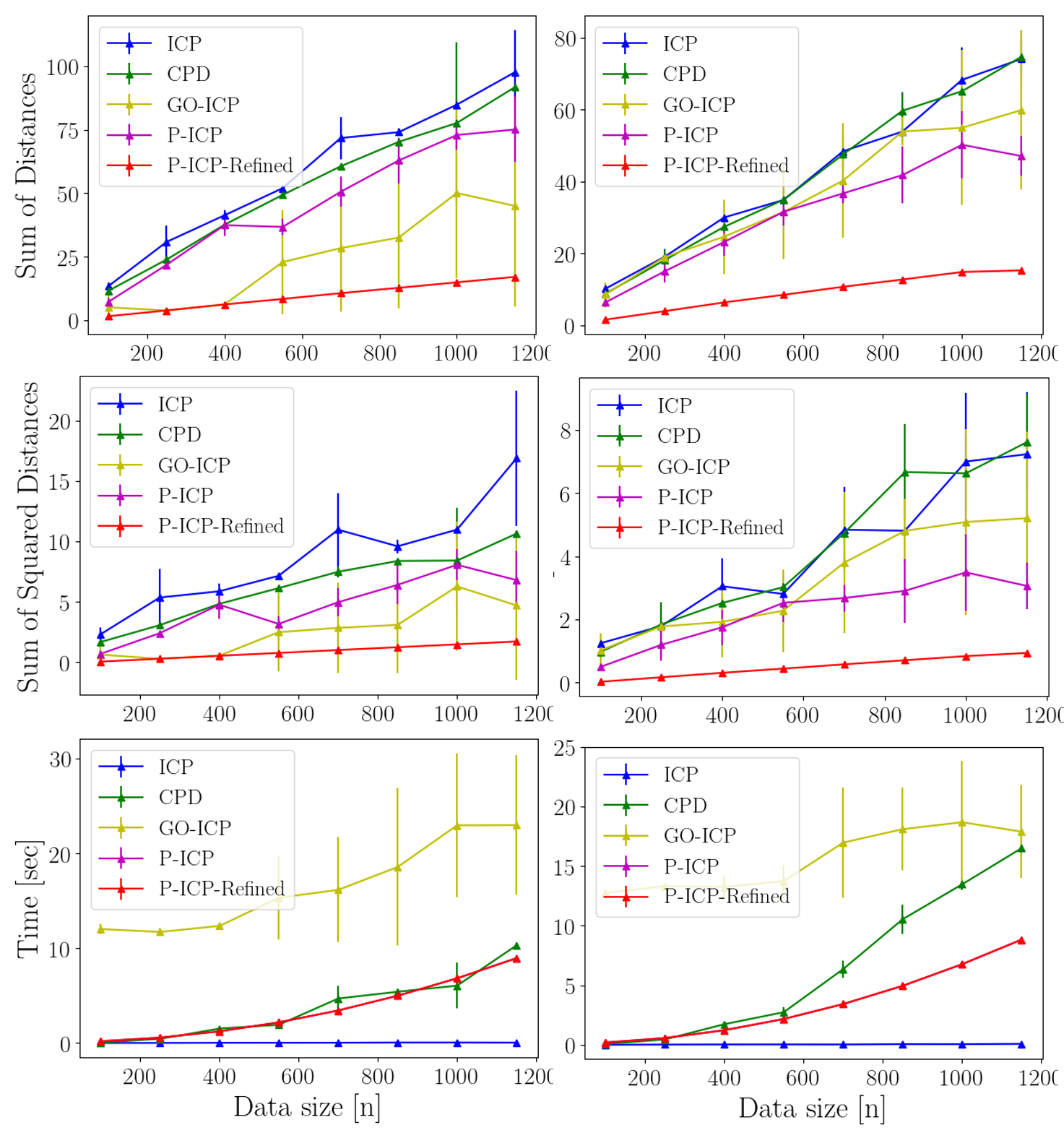}
  \caption{Evaluating the final cost of various algorithms, with multiple cost functions. (Left column): The Armadillo model. (Right column): The Asian Dragon model. The last row presents the computation time it took each algorithm to output its final alignment. $\sigma^2 = 0.01$ was used in all tests. } \label{fig:registrationMain}
\end{figure}

\begin{figure}
  \centering
  \includegraphics[width=0.49\textwidth]{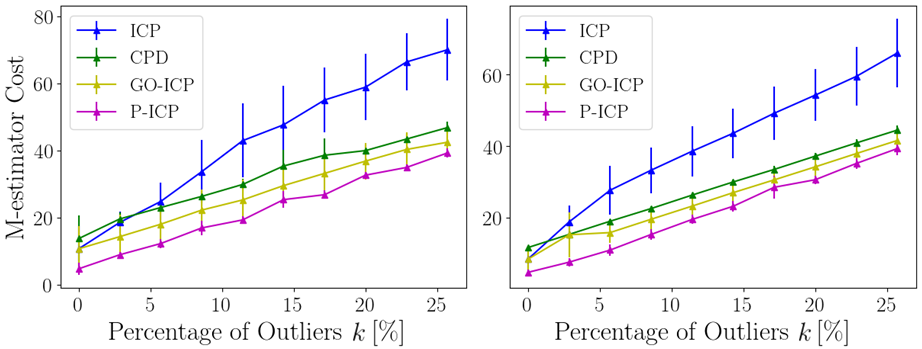}
  \caption{Robustness to outliers. We used the SSD with M-estimator $\lip\left(D(p,q)\right) = \min\{\norm{p-q}^2, 0.2\}$. (Left): Armadillo model. (Right): Asian Dragon model. $n=800$ for both. Noise with variance $\sigma^2 = 1$ was added to $k$ percentage of the points in $P$, which are considered as outliers.} \label{fig:registrationOutliers}
\end{figure}

\textbf{Discussion. }Fig.~\ref{fig:registrationMain} demonstrates the accuracy of our algorithms, which yield an error smaller by up to x$10$ time than other methods, while also being among the fastest. By using more CPUs, our computational time could have been reduced even further. This is in contrast to the competing methods, which are difficult to run in parallel. Fig.~\ref{fig:registrationOutliers} and~\ref{fig:ArmadilloVis} demonstrate our robustness to outliers in practice, due to our provable approximation to M-estimators cost functions.

\section{Conclusions and Future Work}
We present the first provable non-trivial approximation algorithms for the alignment and registration problems and their hard variants. The algorithms rely on our proof that a witness set, which determines an approximated alignment, exists for both problems. Experiments show that our algorithms are efficient in practice, produce smaller errors, and are more stable than competing methods. Future work includes: (i) handling sets of other geometric shapes e.g., sets of lines, (ii) generalizing our results to the non-rigid registration problem, (iii) applying a deep learning based approach for fast recovery of a witness set, and (iv) applying compression techniques such as coresets to reduce the running time of our algorithms.

{\small
\bibliographystyle{ieee_fullname}
\bibliography{references_pc}
}

\clearpage
\newpage

\appendix
\section{Approximation for the Alignment Problem}

In what follows, for $p\in \REAL^d$ we denote by $\proj(p,X)$ its projection on a set $X\subseteq\REAL^d$, that is, $\proj(p,X) \in \argmin_{x \in X} \dist(p,x)$. For a set of vectors $P \subseteq \REAL^d$, we denote by $\linspan(P)$ the linear span of the set $P$. For a single vector $p \in \REAL^d$ we abuse notation and denote $\linspan(\br{p})$ by $\linspan(p)$ for short.

\paragraph{Overview of Claim~\ref{claimRot}. }Let $R \in \R_X$ be an arbitrary rotation matrix that rotates a 2-dimensional plane $X$ in $\REAL^d$, for example, a rotation matrix that affects only the first two coordinates of the points it multiplies.
In this example, $X$ is the $xy$-plane.
In what follows we formally prove that displacement of a unit vector $\hat{p}\in X$, after multiplication with $R$ (i.e., $\norm{R\hat{p}-\hat{p}}$), must be greater than or equal to the displacement of a unit vector $\hat{q} \not\in X$ after multiplication with the same $R$ (i.e., $\norm{R\hat{q}-\hat{q}}$).

\begin{claim} \label{claimRot}
Let $d\geq 2$ be an integer. Let $X$ be a $2$-dimensional subspace (plane) of $\REAL^d$, let $p\in X \setminus \br{\vec{0}} $, and let $R \in \R_{X}$ be a rotation matrix that rotates every $x\in X$ by at most $\theta \in [-\pi/2,\pi/2]$ radians around the origin. Then for every $q\in \REAL^d \setminus \br{\vec{0}}$
\[
\frac{\norm{Rq-q}}{\norm{q}} \leq \frac{\norm{Rp-p}}{\norm{p}}.
\]
\end{claim}
\begin{proof}
Without loss of generality, assume that $X$ is spanned by the standard vectors $e_1,e_2 \in \REAL^d$. Otherwise rotate the coordinates system.
Since $R\in \R_X$ and $X$ is spanned by the first two standard basis vectors, $R$ can be expressed as
\[
R = \left( \begin{array}{ccc}
		R' & \boldsymbol{0} \\
		\boldsymbol{0} & I_{d-2} \end{array} \right),
\]
where $R' = \left( \begin{array}{ccc}
		\cos{\alpha} & \sin{\alpha} \\
	-\sin{\alpha} & \cos{\alpha}\end{array} \right)$
is a two dimensional rotation matrix, for some $\alpha \in [-\pi/2,\pi/2]$;
see Definition~\ref{defRq}.

Let $p' \in \REAL^2$ denote the first two entries of $p$, and $q' \in \REAL^2$ denote the first two entries of $q$. Notice that
\begin{equation} \label{eqnormpnormp2}
\norm{p} = \norm{p'}
\end{equation}
since $p \in X$ and $X$ is spanned by $e_1$ and $e_2$.
By the definition of the matrix $R$, we have that
\begin{equation} \label{eq2Drot}
\norm{Rp - p} = \norm{R'p' -p'} \text{ and } \norm{Rq - q} = \norm{R' q'-q'}.
\end{equation}

Therefore,
\begin{align}
\frac{\norm{Rq-q}^2}{\norm{q}^2}
& = \frac{\norm{R'q'-q'}^2}{\norm{q}^2} \label{eqp11}
= \frac{2\norm{q'}^2 - 2{q'}^T(R'q')}{\norm{q}^2}\\
& \leq \frac{2\norm{q'}^2 - 2{q'}^T(R'q')}{\norm{q'}^2}
= \frac{2\norm{p'}^2 - 2{p'}^T(R'p')}{\norm{p'}^2} \label{eqp12}\\
& = \frac{\norm{R'p'-p'}^2}{\norm{p'}^2}
= \frac{\norm{R'p'-p'}^2}{\norm{p}^2}\label{eqp13}\\
& = \frac{\norm{Rp-p}^2}{\norm{p}^2},\label{eqp14}
\end{align}
where the first derivation in~\eqref{eqp11} is by~\eqref{eq2Drot}, the first derivation in~\eqref{eqp12} holds since $\norm{q'} \leq \norm{q}$, the second derivation in~\eqref{eqp12} holds since $\frac{q'^T(R'q')}{\norm{q'}^2} = \frac{p'^T(R'p')}{\norm{p'}^2}$, the first inequality in~\eqref{eqp13} holds is by~\eqref{eqnormpnormp2}, and~\eqref{eqp14} is by~\eqref{eq2Drot}.

Claim~\ref{claimRot} now holds by taking the squared root of~\eqref{eqp13}.
\end{proof}

\paragraph{Overview of Lemma~\ref{lemRotApprox}.}
For two vectors $p,q \in \REAL^d$, we define the \emph{angle between $p$ and $q$} as the smallest angle between them when considering the $2$-dimensional subspace that they span.

Now, consider two ordered sets of points $P$ and $Q$ in $\REAL^d$, and let $p\in P$ and $q\in Q$ be a pair of corresponding points, whose angle between them is the smallest among all $n$ corresponding pairs. Let $R \in \rot(d)$ be a rotation matrix that aligns the direction of $p$ with the direction of $q$, i.e., $Rp \in \linspan{(q)}$. Then the following lemma proves that after rotating $P$ by $R$, the distance from every point in $P$ to its corresponding point in $Q$ will at most increase by a multiplicative factor of $(1+\sqrt{2})$.
\begin{lemma} \label{lemRotApprox}
Put $r \in [d]$. Let $\pi$ be an $r$-dimensional subspace of $\REAL^d$, $P = \br{p_1,\cdots,p_n} \subset \pi\setminus\br{\vec{0}}$ and $Q = \br{q_1,\cdots,q_n} \subset \pi\setminus\br{\vec{0}}$ be two ordered sets of points. Put $\R^*in \rot(d)$.
Then there is an index $j\in [n]$ and a rotation matrix $R' \in \R_\pi$ that satisfy both of the following properties:
\renewcommand{\labelenumi}{(\roman{enumi})}
\begin{enumerate}
\item For every $i\in [n]$,
\begin{equation} \label{approxRotationToProve_i}
\norm{R'p_i-q_i} \leq (1+\sqrt{2})\cdot \norm{R^*p_i-q_i}
\end{equation} and
\item $R'p_j \in \linspan(q_j)$.
\end{enumerate}
\end{lemma}
\begin{proof}
Throughout this proof, for simplicity of notation, we assume that the points of $P$ have already been rotated by the rotation matrix $R^*$, i.e., we assume that $R^*$ is the identity matrix. Therefore,  \eqref{approxRotationToProve_i} reduces to
\[
\norm{R'p_i-q_i} \leq (1+\sqrt{2})\cdot \norm{p_i-q_i}
\]
for every $i\in [n]$.

In what follows we pick a pair of corresponding input points $p_j,q_j$ which have the smallest angle between them around the origin, among all the input pairs (ties broken arbitrarily). We then show that there is a rotation matrix $R_j$ that aligns the direction vectors of $p_j$ and $q_j$ and satisfies the requirements of the lemma.

For every $i\in [n]$, let $X_i = \linspan(\br{p_i,q_i})$ be the plane spanned by $p_i$ and $q_i$, and $R_i \in \R_{X_i}$ be a rotation matrix that satisfies $R_ip_i \in \linspan(q_i)$, i.e., aligns the directions of the vectors $p_i$ and $q_i$ by a rotation in the $2$-dimensional subspace (plane) $X_i$ that these pair of vectors span. If there is more than one such rotation matrix, pick an arbitrary one among the (possible two) which rotate $p_i$ with the smallest angle of rotation.

Put $i\in [n]$.
Let $\pi_i$ be a $2$-dimensional subspace that contains $p_i$ and $R_ip_i$. By the definition of $R_i$, we have that $R_ip_i \in \linspan(q_i) \subseteq \pi_i$.
Let $j \in \displaystyle \argmin_{k\in [n]} \frac{\norm{R_kp_k-p_k}}{\norm{p_k}}$, i.e., $j$ is the index of the corresponding pair that have the smallest angle among all pair of corresponding points.

We now prove that the distance $\norm{R_jp_i-q_i}$ between the corresponding pair $p_i \in P$ and $q_i \in Q$ after applying $R_j$ is larger by at most a multiplicative factor of $(1+\sqrt{2})$ compared to their original distance $\norm{p_i-q_i}$.
We have that
\begin{equation} \label{eqRpiRpj}
\frac{\norm{R_jp_i-p_i}}{\norm{p_i}} \leq \frac{\norm{R_jp_j-p_j}}{\norm{p_j}} \leq \frac{\norm{R_ip_i-p_i}}{\norm{p_i}},
\end{equation}
where the first inequality holds by substituting $p=p_j, q = p_i$ and $R = R_j$ in Claim~\ref{claimRot}, and the second inequality holds by the definition of $j$.
Multiplying~\eqref{eqRpiRpj} by $\norm{p_i}$ yields
\begin{equation} \label{eqMinRot}
\norm{R_jp_i-p_i} \leq \norm{R_ip_i-p_i}.
\end{equation}

We now prove Lemma~\ref{lemRotApprox} (i) for $R' = R_j$, i.e., we prove that
\begin{equation} \label{eqtoprove}
\norm{R_jp_i-q_i} \leq (1+\sqrt{2}) \cdot \norm{p_i-q_i}.
\end{equation}
Lemma~\ref{lemRotApprox} (ii) then follows by the definition of $R_j$.

We prove~\eqref{eqtoprove} by the following case analysis: \textbf{(i)} $R_i$ is the identity matrix, \textbf{(ii)} $\norm{q_i} > \norm{p_i}$ and $R_i$ is not the identity matrix, and \textbf{(iii)} $\norm{q_i} \leq \norm{p_i}$ and $R_i$ is not the identity matrix.

\textbf{Case (i): }$R_i$ is the identity matrix. In this case we have that $\norm{R_ip_i-p_i} = 0$. Therefore, by the definition of $j$, we have that $\norm{R_jp_j-p_j}=0$, i.e., $R_j$ is also the identity matrix. Hence, Case (i) trivially holds as $\norm{R_jp_i-q_i} = \norm{p_i-q_i} \leq (1+\sqrt{2})\cdot\norm{p_i-q_i}$.

\textbf{Case (ii): }$\norm{q_i} > \norm{p_i}$ and $R_i$ is not the identity matrix.
By the definition of $\proj$, we have that $p_i - \proj(p_i,\linspan(q_i))$ is orthogonal to $q_i$. Combining this with the fact that the vector $(\proj(p_i,\linspan(q_i))-q_i)$ has the same direction as $q_i$ yields that $(p_i-\proj(p_i,\linspan(q_i)))$ is orthogonal to $(\proj(p_i,\linspan(q_i))-q_i)$, i.e.,
\begin{equation}\label{eqOrtho}
(p_i-\proj(p_i,\linspan(q_i)))^T (\proj(p_i,\linspan(q_i))-q_i) = 0.
\end{equation}
Similarly, we have that $(p_i-\proj(p_i,\linspan(q_i)))$ is orthogonal to $(\proj(p_i,\linspan(q_i))-R_ip_i)$. Therefore, by applying the Pythagorean theorem in the right triangle $\Delta(p_i,\proj(p_i,\linspan(q_i)),R_ip_i)$ we obtain
\begin{equation} \label{eq:Pyth1}
\begin{split}
\norm{p_i-\proj(p_i,\linspan(q_i))}^2 & + \norm{\proj(p_i,\linspan(q_i))-R_ip_i}^2\\
&= \norm{p_i-R_ip_i}^2.
\end{split}
\end{equation}

Observer that
\begin{equation} \label{eq1314}
\norm{\proj(p_i,\linspan(q_i))} \leq \norm{p_i} = \norm{R_ip_i} < \norm{q_i},
\end{equation}
where the last derivation is by the assumption of Case (ii). Combining~\eqref{eq1314} with the fact that $\proj(p_i,\linspan(q_i)), R_ip_i$ and $q_i$ lie on the straight line $\linspan(q_i)$, we obtain that
\begin{equation}\label{eqdivqi}
\begin{split}
& \norm{\proj(p_i,\linspan(q_i))-q_i}\\
& = \norm{\proj(p_i,\linspan(q_i))-R_ip_i}+ \norm{R_ip_i-q_i}.
\end{split}
\end{equation}

Therefore
\begin{align}
\norm{p_i-q_i}^2 & = \norm{p_i-\proj(p_i,\linspan(q_i))+\proj(p_i,\linspan(q_i))-q_i}^2 \nonumber\\
& = \norm{p_i-\proj(p_i,\linspan(q_i))}^2 \nonumber\\ &\quad + \norm{\proj(p_i,\linspan(q_i))-q_i}^2 \label{secondDiv}\\
& = \norm{p_i-\proj(p_i,\linspan(q_i))}^2 \nonumber\\ &\quad+ \left(\norm{\proj(p_i,\linspan(q_i))-R_ip_i}+ \norm{R_ip_i-q_i}\right)^2 \label{thirdDiv}\\
& = \norm{p_i-\proj(p_i,\linspan(q_i))}^2 \nonumber\\ &\quad + \norm{\proj(p_i,\linspan(q_i))-R_ip_i}^2 + \norm{R_ip_i-q_i}^2 \nonumber\\&\quad +2\norm{\proj(p_i,\linspan(q_i))-R_ip_i}\cdot \norm{R_ip_i-q_i} \nonumber\\
& = \norm{R_ip_i-p_i}^2 + \norm{R_ip_i-q_i}^2 \nonumber\\ &\quad+ 2\norm{\proj(p_i,\linspan(q_i))-R_ip_i}\cdot \norm{R_ip_i-q_i} \label{eq:byPyth1}\\
& \geq \norm{R_ip_i-p_i}^2, \label{eqRipiqi}
\end{align}
where~\eqref{secondDiv} is by~\eqref{eqOrtho}, \eqref{thirdDiv} is by~\eqref{eqdivqi}, and~\eqref{eq:byPyth1} is by~\eqref{eq:Pyth1}.
Combining~\eqref{eqMinRot} and~\eqref{eqRipiqi} yields
\begin{equation} \label{eqRjpiqi}
\norm{R_jp_i-p_i} \leq \norm{p_i-q_i}.
\end{equation}

Hence, \eqref{eqtoprove} holds for Case (ii) as
\[
\norm{R_jp_i-q_i} \leq \norm{R_jp_i-p_i} + \norm{p_i-q_i} \leq 2\cdot \norm{p_i-q_i},
\]
where the first derivation holds by the triangle inequality, and the second derivation by~\eqref{eqRjpiqi}.

\textbf{Case (iii): }$\norm{q_i} \leq \norm{p_i}$ and $R_i$ is not the identity matrix. Since $R_i$ is not the identity matrix, then $p_i$ and $R_ip_i$ are distinct points.
Let $q^* = \proj(p_i,\linspan(R_ip_i))$. Combining the definition of $q^*$ with the fact that $p_i$ and $R_ip_i$ are distinct points, we obtain that all three points $p_i,R_ip_i$ and $q^*$ are distinct.

Consider the triangle $\Delta(p_i,R_ip_i,\vec{0})$.
Let $\alpha\in [0,\pi/2]$ be the angle at vertex $\vec{0}$ of the triangle; see Fig.~\ref{fig:sinLaw} for illustration.
Since $\norm{p_i} = \norm{R_ip_i}$ we obtain that the angles at vertices $p_i$ and $R_ip_i$ is $\beta = \frac{\pi-\alpha}{2}$. Consider triangle $\Delta(p_i,R_ip_i,q^*)$, and observe that the angle at vertex $q^*$ is $\pi/2$. It therefore holds that
\begin{equation} \label{eqSines}
\begin{split}
\norm{R_ip_i-p_i} & = \frac{\norm{p_i-q^*} \sin \pi/2}{\sin{\frac{\pi-\alpha}{2}}} = \frac{\norm{p_i-q^*}}{\sin{\frac{\pi-\alpha}{2}}}\\
&\leq \frac{\norm{p_i-q^*}}{\sin{\frac{\pi-\pi/2}{2}}} = \sqrt{2}\cdot \norm{p_i-q^*} \leq \sqrt{2}\cdot \norm{p_i-q_i},
\end{split}
\end{equation}
where the firsst derivation is by applying the law of sines in triangle $\Delta(p_i,R_ip_i,q^*)$, and the last derivation is by the definition of $q^*$.

Hence, \eqref{eqtoprove} holds for Case (iii) as
\begin{equation} \label{eqPropi}
\begin{split}
\norm{R_jp_i-q_i} & \leq \norm{R_jp_i-p_i} + \norm{p_i-q_i}\\
& \leq \norm{R_ip_i-p_i} + \norm{p_i-q_i}\\
& \leq (1+\sqrt{2})\cdot \norm{p_i-q_i},
\end{split}
\end{equation}
where the fist derivation holds by the triangle inequality, the second derivation is by~\eqref{eqMinRot}, and the last derivation holds by~\ref{eqSines}.

\begin{figure}
  \centering
  \includegraphics[width=0.3\textwidth]{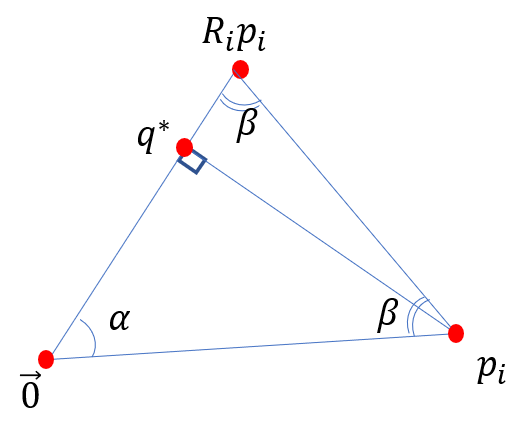}
  \caption{A triangle $\Delta(p_i,R_ip_i,\vec{0})$, a point $q^* = \proj(p_i,\linspan(R_ip_i))$ and a right triangle $\Delta(p_i,q^*,\vec{0})$.}\label{fig:sinLaw}
\end{figure}

\end{proof}

The following lemma states that given two ordered sets $P,Q \subset \REAL^d$, and every rotation matrix $R^*$, we can compute a rotation matrix $R'$ such that $\norm{R'p-q} \leq c\cdot \norm{R^*p-q}$ for every matched pair of points $p\in P$ and $q\in Q$, up to some multiplicative constant factor $c$. Furthermore, if the two sets $P$ and $Q$ are contained in some $r$-dimensional subspace $\pi$ of $\REAL^d$, then $R' \in \R_{\pi}$, i.e., when applying $R'$ on the set of points $P$, the resulting set of rotated points is also contained in $\pi$.
\begin{lemma} \label{lemApproxD}
Let $r \in [d]$ be an integer and $P = \br{p_1,\cdots,p_n} \subset \REAL^d \setminus\br{\vec{0}}$, $Q = \br{q_1,\cdots,q_n} \subset \REAL^d\setminus\br{\vec{0}}$ be two ordered sets of points that are contained in an $r$-dimensional subspace $\pi$ of $\REAL^d$. Put $R^* \in \rot(d)$. Let
\begin{equation} \label{appendAllRotations}
A = \bigcup_{i_1,\cdots,i_{r-1} \in [n]} \algnamegetrot(p_{i_1},\cdots ,p_{i_{r-1}},q_{i_1},\cdots ,q_{i_{r-1}});
\end{equation}
see Algorithm~\ref{a106}. Then there is $\hat{R} \in A$ such that $\hat{R} \in \R_{\pi}$ and
\[
\norm{\hat{R}p_i-q_i} \leq (1+\sqrt{2})^{r-1} \cdot \norm{R^*p_i-q_i}, \quad \text{for every } i \in [n],
\]
\end{lemma}
\begin{proof}
Put $i\in [n]$ and $R^* \in \rot(d)$. Without loss of generality assume that $R^*$ is the identity matrix and that $\pi$ is spanned by $e_1,\cdots,e_r$, otherwise rotate the set $P$ of points and rotate the coordinates system respectively.
The proof is by induction on $r$.

\textbf{Base case for $r=2$. }By substituting $r=2$ in Lemma~\ref{lemRotApprox} (i) and (ii), we obtain that there is $k\in [n]$ and $R'\in \R_\pi$ that satisfy the following pair of properties.
\renewcommand{\labelenumi}{(\roman{enumi})}
\begin{enumerate}
\item $\norm{R'p_i-q_i} \leq (1+\sqrt{2})\cdot \norm{p_i-q_i}$ for every $i\in [n]$.
\item $R'p_k \in \linspan(q_k)$.
\end{enumerate}

In this base case we iterate over every $i_1\in [n]$ and make a call to $\algnamegetrot(p_{i_1},q_{i_1})$ and then combine the outputs.
In Line~\ref{line:compR} of Algorithm~\ref{a106} we compute a rotation matrix $R\in \R_\pi$ that satisfies that $Rp_{i_1} \in \linspan(q_{i_1})$. Hence, once $i_1=k$, we are guaranteed to compute the required rotation matrix $R'\in \R_\pi$ that satisfies $R'p_k \in \linspan(q_k)$, then the algorithm will terminate since $j=r=2$. Therefore, for the case of where $r=2$ we have
\[
\norm{R'p_i-q_i} \leq (1+\sqrt{2}) \cdot \norm{p_i-q_i} = (1+\sqrt{2}) \cdot \norm{R^*p_i-q_i},
\]
and $R' \in \R_\pi$.

\textbf{General case for $r=r'+1 > 2$. }Let us assume that Lemma~\ref{lemApproxD} holds for $r = r' \geq 2$. We now prove that Lemma~\ref{lemApproxD} holds for $r = r'+1$. Informally, we do so by first invoking Lemma~\ref{lemRotApprox} to obtain a rotation matrix $R'$ that aligns one of the points $p_k \in P$ with its corresponding $q_k \in Q$. By the first property of Lemma~\ref{lemRotApprox}, for every $p\in P$, applying $R'$ to $p$ does not increase the distance $\norm{p-q}$ from $p$ to its corresponding $q\in Q$ by more than a multiplicative constant. Therefore, in Algorithm~\ref{a106}, we first apply $R'$ to the points of $P$. Then, we project the points onto a smaller dimensional space, and compute a rotation matrix $R''$ that aligns the projected points of $P$ to $Q$ in this lower dimensional subspace. From the inductive assumption we know that $R''$ does not increase the pairwise distances by more than a multiplicative constant factor. Combining $R'$ and $R''$ concludes this proof.

Now formally, by plugging $P$ and $Q$ in Lemma~\ref{lemRotApprox}, we obtain that there exists $k\in [n]$ and $R' \in \R_\pi$ that satisfy the following pair of properties:
\begin{equation} \label{eqRpropi}
\norm{R'p_i-q_i} \leq (1+\sqrt{2})\cdot \norm{p_i-q_i}, \quad \text{for every } i \in [n]
\end{equation}
and
\[
R'p_k \in \linspan(q_k).
\]
\renewcommand{\labelenumi}{(\roman{enumi})}

Let $p_i' = R'p_i$.
Assume without loss of generality that $q_k$ spans the $x$-axis, otherwise rotate the coordinates system.
Let $\pi' \subset \pi$ be an $(r-1)$-dimensional subspace that is orthogonal to the $x$-axis and passes through the origin.
Let $\pi_i \subset \pi$ be an $(r-1)$-dimensional affine subspace that is orthogonal to the $x$-axis and passes through $p_i'$ and let $q_i' = \proj(q_i,\pi_i)$; see Fig.~\ref{figProj} for visualization.

For every $p\in \pi_i$ and rotation matrix $R\in \R_{\pi'}$, we have
\begin{equation} \label{eqDistTria}
\begin{split}
\norm{Rp-q_i}^2 & = \norm{Rp-q_i'}^2 + \norm{q_i'-q_i}^2\\
& = \norm{Rp-q_i'}^2 + \dist(q_i,\pi_i)^2\\
& = \norm{\proj(Rp,\pi')-\proj(q_i',\pi')}^2 + \dist(q_i,\pi_i)^2\\
& = \norm{\proj(Rp,\pi')-\proj(q_i,\pi')}^2 + \dist(q_i,\pi_i)^2\\
& = \norm{R \proj(p,\pi')-\proj(q_i,\pi')}^2 + \dist(q_i,\pi_i)^2,
\end{split}
\end{equation}
where the first derivation holds by the Pythagorean theorem, the second derivation holds by the definition of $q_i'$, the third derivation holds by combining that $Rp,q_i' \in \pi_i$ and that $\pi_i$ and $\pi'$ are parallel, the fourth equality holds since $\proj(q_i',\pi') = \proj(q_i,\pi')$ and the last derivation holds by combining that $R \in \R_{\pi'}$ and that $\pi_i$ and $\pi'$ are parallel; see Fig.~\ref{figProj} for a visualization of the derivations in~\eqref{eqDistTria}.

\begin{figure}
  \centering
  \includegraphics[width=0.4\textwidth]{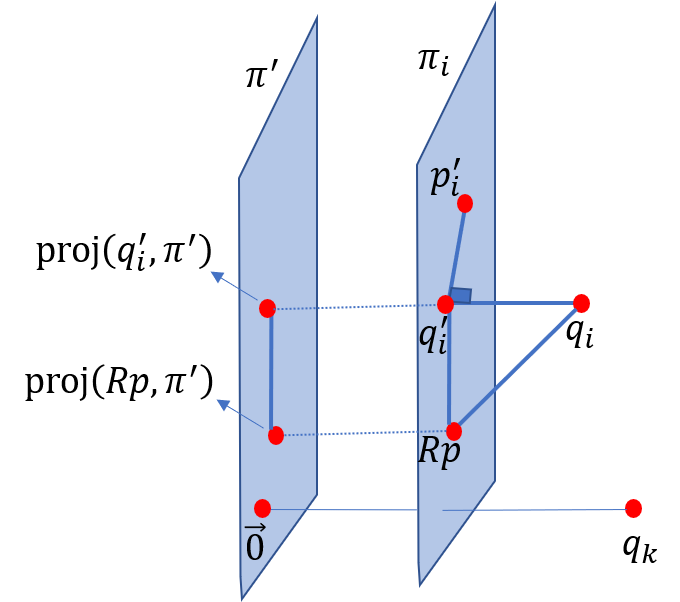}
  \caption{Illustration of the derivations in Eq.~\eqref{eqDistTria}.}\label{figProj}
\end{figure}


Let $a_i = \proj(p_i',\pi')$ and $b_i = \proj(q_i,\pi')$ for every $i\in [n]$.
Observe that $\br{a_1,\cdots,a_n}$ and $\br{b_1,\cdots,b_n}$ are contained in the $(r-1)$-dimensional subspace $\pi'$.
Let
\[
A' = \bigcup_{i_1,\cdots,i_{r-1} \in [n]} \algnamegetrot(a_{i_1},\cdots ,a_{i_{r-2}},b_{i_1},\cdots ,b_{i_{r-2}});
\]
see Algorithm~\ref{a106}. By our induction assumption, Lemma~\ref{lemApproxD} holds for $r=r'$. Therefore, for $R^* = I_d$, there exists $R''\in A'$ such that $R'' \in R_{\pi'}$ and
\begin{equation} \label{eqD}
\begin{split}
& \norm{R''\proj(p_i',\pi')-\proj(q_i,\pi'))}\\ &\leq (1+\sqrt{2})^{r'-1} \cdot \norm{\proj(p_i',\pi')-\proj(q_i,\pi')}.
\end{split}
\end{equation}

We thus have that
\begin{align}
&\norm{R''R'p_i-q_i}^2 = \norm{R''p_i'-q_i}^2 \label{eqmain11}\\
& = \norm{R'' \proj(p_i',\pi')-\proj(q_i,\pi')}^2 + \dist^2(q_i,\pi_i)^2 \label{eqmain12}\\
& \leq (1+\sqrt{2})^{2(r'-1)} \cdot \norm{\proj(p_i',\pi')-\proj(q_i,\pi'))}^2\nonumber\\ &+ \dist^2(q_i,\pi_i)^2 \label{eqmain14}
\\
& \leq (1+\sqrt{2})^{2(r'-1)} \cdot \norm{\proj(p_i',\pi')-\proj(q_i,\pi'))}^2 \nonumber\\ &+ (1+\sqrt{2})^{2(r'-1)} \cdot \dist^2(q_i,\pi_i)^2 \nonumber\\
& = (1+\sqrt{2})^{2(r'-1)} \cdot \norm{\proj(p_i',\pi_i)-\proj(q_i,\pi_i))}^2 \nonumber\\ &+ (1+\sqrt{2})^{2(r'-1)} \cdot \dist^2(q_i,\pi_i)^2 \label{eqmain152}\\
& = (1+\sqrt{2})^{2(r'-1)} \cdot \norm{p_i'-\proj(q_i,\pi_i))}^2 \nonumber\\ &+ (1+\sqrt{2})^{2(r'-1)} \cdot \dist^2(q_i,\pi_i)^2 \label{eqmain153}\\
& = (1+\sqrt{2})^{2(r'-1)} \cdot \norm{p_i'-q_i}^2 \label{eqmain16}\\
& = (1+\sqrt{2})^{2(r'-1)} \cdot \norm{R'p_i-q_i}^2 \label{eqmain17}\\
& \leq (1+\sqrt{2})^{2r'} \cdot \norm{p_i-q_i}^2 \label{eqmain18},
\end{align}
where~\eqref{eqmain11} holds by the definition of $p_i'$, \eqref{eqmain12} holds by substituting $p = p_i'$ and $R = R''$ in~\eqref{eqDistTria}, \eqref{eqmain14} holds by squaring both sides of~\eqref{eqD}, \eqref{eqmain152} holds since $\pi'$ is parallel to $\pi_i$, \eqref{eqmain153} holds since $p_i' \in \pi_i$, \eqref{eqmain16} is by the Pythagorean theorem, \eqref{eqmain17} holds by the definition of $p_i'$, and \eqref{eqmain18} holds by~\eqref{eqRpropi}; see Fig.~\ref{figProj}.

Let $\hat{R} = R''R'$. Hence, it follows that
\begin{equation} \label{eqconclude1}
\norm{R''R'p_i-q_i} \leq (1+\sqrt{2})^{r'} \cdot \norm{p_i-q_i} = (1+\sqrt{2})^{r-1} \cdot \norm{p_i-q_i}.
\end{equation}
By combining~ that $R' \in \R_\pi$ and $R''\in \R_{\pi'} \subseteq \R_{\pi}$, it also follows that
\begin{equation} \label{eqconclude2}
\hat{R} \in \R_{\pi}.
\end{equation}

Observe that whenever $i_1 = k$ in~\eqref{appendAllRotations}, a matrix $R'$ that satisfies $R'p_k \in \linspan(q_k)$ is computed in Line~\ref{line:compR} of Algorithm~\ref{a106}. In Line~\ref{lineRecursieRot} of Algorithm~\ref{a106} we compute the sets $P'$ and $Q'$ and recursively call $\algnamegetrot(P',Q')$. By the definition of the rotation matrix $R''$, one of the calls to $\algnamegetrot(P',Q')$ is guaranteed to compute $R''$. Hence, when $i_1=k$, Algorithm~\ref{a106} will compute a rotation matrix $R''R'$ and return it as output.
Hence, by the definition of $A$ in~\eqref{appendAllRotations} we are guaranteed that
\begin{equation} \label{eqconclude3}
\hat{R} = R''R' \in A.
\end{equation}
The running time of the algorithm is $n^{O(d)}$ since there are $n^{O(d)}$ calls to Algorithm~\ref{a106}, each requiring $O(d)$ time.
Hence, Lemma~\ref{lemApproxD} holds by combining~\eqref{eqconclude1},~\ref{eqconclude2} and~\ref{eqconclude3}.

\end{proof}

\begin{theorem} [Theorem~\ref{lemApproxAlignment}] \label{lemApproxAlignment_proof}
Let $P = \br{p_1,\cdots,p_n}$ and $Q = \br{q_1,\cdots,q_n}$ be two ordered sets each of $n$ points in $\REAL^d$.
Let $M \subseteq \alignments(d)$ be the output of a call to \algname$(P,Q)$; see Algorithm~\ref{a107}.
Then, for every $(R^*,t^*) \in \alignments(d)$, there exists $(\hat{R},\hat{t}) \in M$ that satisfies the following for every $i\in [n]$,
\[
\norm{\hat{R}p_i-\hat{t}-q_i} \leq (1+\sqrt{2})^{d} \cdot \norm{R^*p_i-t^*-q_i}.
\]
Furthermore, $(\hat{R},\hat{t})$ is computed in $n^{O(d)}$ time.
\end{theorem}
\begin{proof}
Put $(R^*,t^*)\in \alignments(d)$. Without loss of generality assume that $R^*$ is the identity matrix of $\REAL^d$ and that $t^* = \vec{0}$, otherwise rotate and translate the coordinates system.

For every $i\in [n]$ let $t_i = p_i-q_i$ and let $k\in \argmin_{i\in [n]} \norm{t_i}$. Put $i\in [n]$.
Since $\norm{t_k} \leq \norm{t_i} = \norm{p_i-q_i}$, it holds that
\begin{equation} \label{eqTransApprox}
\begin{split}
\norm{p_i-t_k-q_i} & \leq \norm{p_i-t_k-p_i} + \norm{p_i-q_i}\\
& = \norm{t_k} + \norm{p_i-q_i}\\
& \leq  2\cdot \norm{p_i-q_i}.
\end{split}
\end{equation}

Let $p_i' = p_i-p_k$, $q_i' = q_i-q_k$, $P' = \br{p_1',\cdots,p_n'}$ and $Q' = \br{q_1',\cdots,q_n'}$.
Consider the set $A_i$ computed at Line~\ref{lineGetRot}.
By Lemma~\ref{lemApproxD} there is $\hat{R} \in A_i$ such that
\begin{equation} \label{prop1}
\begin{split}
\norm{\hat{R}p_i'-q_i'}&  = \norm{\hat{R}p_i'-q_i'}\\
& \leq (1+\sqrt{2})^{d-1} \cdot \norm{\idmat(d)p_i'-q_i'}\\
& = (1+\sqrt{2})^{d-1} \cdot \norm{p_i'-q_i'}.
\end{split}
\end{equation}

Let $\hat{t} = t_k + \hat{R}p_k -p_k$. We now obtain that
\begin{align}
\norm{\hat{R}p_i-\hat{t}-q_i} & = \norm{\hat{R}p_i-t_k-\hat{R}p_k+p_k-q_i} \label{eq11}\\
& = \norm{\hat{R}(p_i-p_k)-(p_k-q_k)+p_k-q_i} \nonumber\\
& = \norm{\hat{R}(p_i-p_k)+q_k-q_i} \nonumber\\
& = \norm{\hat{R}p_i'-q_i'} \label{eq13}\\
& \leq (1+\sqrt{2})^{d-1} \cdot \norm{p_i'-q_i'} \label{eq14}\\
& = (1+\sqrt{2})^{d-1} \cdot \norm{p_i-p_k-q_i+q_k} \nonumber\\
& = (1+\sqrt{2})^{d-1} \cdot \norm{p_i-p_k + q_k-q_i} \nonumber\\
& = (1+\sqrt{2})^{d-1} \cdot \norm{p_i-t_k-q_i} \nonumber\\
& \leq 2(1+\sqrt{2})^{d-1} \cdot \norm{p_i-q_i} \label{eq16},
\end{align}
where~\eqref{eq11} holds by the definition of $\hat{t}$, \eqref{eq13} holds by the definition of $p_i'$ and $q_i'$, \eqref{eq14} holds by~\eqref{prop1}, and~\eqref{eq16} holds by~\eqref{eqTransApprox}.

Hence, it holds that
\begin{equation} \label{eqres}
\norm{\hat{R}p_i-\hat{t}-q_i} \leq 2(1+\sqrt{2})^{d-1}\cdot \norm{p_i-q_i} \leq (1+\sqrt{2})^{d}\cdot \norm{p_i-q_i}.
\end{equation}

In Line~\ref{lineESalignment} of Algorithm~\ref{a107} we go over every index $i\in [n]$. In Line~\ref{lineGetRot} we compute a set of rotation matrices $A_i$ that align the sets $P'=\br{p_1-p_i,\cdots,p_n-p_i} \setminus \br{\vec{0}},Q'=\br{q_1-q_i,\cdots,q_n-q_i} \setminus \br{\vec{0}}$ by a call to Algorithm~\ref{a106}, and in Line~\ref{linecompt} we compute a set of pairs that consist of a rotation matrix $R \in R_i$ and its corresponding translation vector $Rp_i-q_i$.
When $i=k$, we are guaranteed to compute the desired pair $(\hat{R},\hat{t})$ and add it to the set $M$.
Therefore, Theorem~\ref{lemApproxAlignment} holds by combining~\eqref{eqres} and $(\hat{R},\hat{t})\in M$.

The running time of Algorithm~\ref{a107} is $n^{O(d)}$ since we make $O(n^{r})$ calls to Algorithm~\ref{a106} at Line~\ref{lineGetRot}, each call takes $O(d^2)$ time.
\end{proof}

\section{Generalization}

In what follows is a formal definition of an $r$-log-Lipschitz function, which is a generalization of the definition in~\cite{feldman2012data} from $1$ to $n$ dimensional functions. Intuitively, an $r$-log-Lipschitz function is a function whose derivative may be large but cannot increase too rapidly (in a rate that depends on $r$). For a higher dimensional function, we demand the previous constraint over every dimension individually.

For every pair of vectors $v=(v_1,\cdots,v_n)$ and $u=(u_1,\cdots,u_n)$ in $\REAL^n$ we denote $v\leq u$ if $v_i \leq u_i$ for every $i\in [n]$. Similarly, $f:\REAL^n \to [0,\infty)$ is non-decreasing if $f(v)\leq f(u)$ for every $v\leq u \in \REAL^d$.
\begin{definition}[Log-Lipschitz function] \label{def:lip}
Let $n\geq 1$, $r>0$ and let $h:[0,\infty)^n \to[0,\infty)$ be a non-decreasing function.
Then $h(x)$ is \emph{$r$-log-Lipschitz} if for every $x \in [0,\infty)$ and $c > 0$ we have that $h(c x)\leq c^r h(x)$. For $n\geq 2$, a function $h:[0,\infty)^n \to [0,\infty)$ is called a
The parameter $r$ is called the \emph{log-Lipschitz constant}.
\end{definition}

The following observation states that if we find an alignment $(R',t') \in \alignments(d)$ that approximates the function $D$ for every input element, then it also approximates the more complex function $\cost$ as defined in Definition~\ref{def:cost}.
\begin{observation} [Observation 5 in~\cite{jubran2020aligning}] \label{obs:distToCost}
Let $P = \br{p_1,\cdots,p_n}$ and $Q = \br{q_1,\cdots,q_n}$ and $\cost$ be as defined in Definition~\ref{def:cost}.
Let $(R^*,t^*),(R',t') \in \alignments(d)$ and let $c\geq 1$. If $D\left(R'p_i-t',q_i\right)\leq c\cdot D\left(R^*p_i-t^*,q_i\right)$ for every $i\in [n]$, then
\[
\cost\left(P,Q,(R',t')\right)\leq c^{rs} \cdot \cost\left(P,Q,(R^*,t^*)\right).
\]
\end{observation}

\begin{corollary} [Corollary~\ref{corApproxCost}] \label{corApproxCost_proof}
Let $P = \br{p_1,\cdots,p_n}$, $Q = \br{q_1,\cdots,q_n}$ be two ordered sets of $n$ points in $\REAL^d$, $z > 0$, and $w = d^{\abs{\frac{1}{z}-\frac{1}{2}}}$.
Let $\cost$, $r$ and $s$ be as defined in Definition~\ref{def:cost} for $D(p,q) = \norm{p-q}_z$.
Let $M \subseteq \alignments(d)$ be the output of a call to \algname$(P,Q)$; See Algorithm~\ref{a107}.
Then, there is $(R',t') \in M$ that satisfies,
\[
\cost(P,Q,(R',t')) \leq w^{rs}\cdot (1+\sqrt{2})^{drs} \cdot \min_{(R,t)}\cost(P,Q,(R,t)),
\]
where the minimum is over every $(R,t) \in \alignments$.
Moreover, $(R',t')$ is computed in $n^{O(d)}$ time.
\end{corollary}
\begin{proof}
Let $(R^*,t^*) \in \argmin_{(R,t)} \cost(P,Q,(R,t))$.
Plugging $P,Q$ and $(R^*,t^*)$ in Theorem~\ref{lemApproxAlignment} yields that in $n^{O(d)}$ time we can compute an alignment $(\hat{R}, \hat{t})$ such that for every $i\in [n]$
\begin{equation} \label{eq:approxl2norm}
\begin{split}
\norm{\hat{R}p_i-\hat{t}-q_i} \leq (1+\sqrt{2})^{d} \cdot \norm{R^*p_i-t^*-q_i}.
\end{split}
\end{equation}

By~\eqref{eq:approxl2norm} and since the $\ell_2$-norm of every vector in $\REAL^d$ is approximated up to a multiplicative factor of $w = d^{\abs{\frac{1}{z}-\frac{1}{2}}}$ by its $\ell_z$-norm, for every $i\in [n]$ we have that
\[
\norm{\hat{R}p_i-\hat{t}-q_i}_z \leq w(1+\sqrt{2})^{d} \cdot \norm{R^*p_i-t^*-q_i}_z.
\]

Corollary~\ref{corApproxCost} now holds by plugging $P, Q, D=D_z, \cost, r,s$, $(R',t')$, $(R^*,r^*)$ and $c=w(1+\sqrt{2})^{d}$ in Observation~\ref{obs:distToCost} as
\[
\begin{split}
& \cost(P,Q,(R',t')) \leq \left(w(1+\sqrt{2})^{d}\right)^{rs} \cost(P,Q,(R^*,t^*))\\ &=  w^{rs}\cdot (1+\sqrt{2})^{drs} \min_{(R,t)\in\alignments(d)} \cost(P,Q,(R^*,t^*)).
\end{split}
\]
\end{proof}

\section{Approximation for the Registration Problem}

\begin{theorem} [Theorem~\ref{theorem:matching}] \label{theorem:matching_proof}
Let $P = \br{p_1,\cdots,p_n}$, $Q = \br{q_1,\cdots,q_n}$ be two ordered sets of $n$ points in $\REAL^d$, $z >0$, and $w = d^{\abs{\frac{1}{z}-\frac{1}{2}}}$.
Let $\cost$ and $r$ be as in Definition~\ref{def:cost} for $D = \norm{p-q}_z$ and $f(v) = \norm{v}_1$. Let $(\tilde{R},\tilde{t},\tilde{\M})$ be the output of a call to \algnamematching$(P,Q,\cost)$; See Algorithm~\ref{AlgMatching}. Then,
\begin{equation}
\begin{split}
& \cost\left(P_{[\tilde{\M}]},Q,(\tilde{R},\tilde{t})\right)\\
& \quad\leq w^r (1+\sqrt{2})^{dr} \cdot \min_{(R,t,\M)}\cost\left(P_{[\M]},Q,(R,t)\right),
\end{split}
\end{equation}
where the minimum is over every alignment $(R,t)$ and permutation $\M$.
Moreover, $(\tilde{R},\tilde{t},\tilde{\M})$ is computed in $n^{O(d)}$ time.
\end{theorem}
\begin{proof}
Let $(R^*,t^*,\M^*) \displaystyle \in \argmin_{(R,t,\M)} \cost(P_{[\M]},Q,(R,t))$.
Assuming $\M^*$ is given, by Corollary~\ref{corApproxCost} we can compute a set $M \subseteq \alignments(d)$ that contains an alignment $(R,t)$ which satisfies that
\begin{equation} \label{eq:Rtnomatching}
\cost\left(P_{[\M^*]},Q,(R,t)\right) \leq w^r (1+\sqrt{2})^{dr} \cdot \cost\left(P_{[\M^*]},Q,(R^*,t^*)\right),
\end{equation}
where $M$ is computed by a call to \algname$(P,Q)$; see Algorithm~\ref{a107}.

The alignment $(R,t)$ above is computed at Line~\ref{linecompt}. Observe that $R$ is computed at Line~\ref{lineGetRot} via a call to Algorithm~\ref{a106}, and is determined by some set of $d-1$ points from $P$ and their corresponding $d-1$ points from $Q$. Observe also that the translation $t = Rp_i-q_i$, computed at Line~\ref{linecompt}, is determined by the rotation matrix $R$ and some matching pair $(p_i,q_i)$. Therefore, each alignment in $M$, and in particular $(R,t)$ above, is determined by a set of $d$ corresponding pairs from $P$ and $Q$.



Let $(p_{i_1},q_{\M^*(i_1)}),\cdots,(p_{i_d},q_{\M^*(i_d)})$ be the $d$ matched pairs of points that correspond to the alignment $(R,t)$. Hence, it holds that
\[
(R,t) \in \algname\left(\br{p_{i_1},\cdots,p_{i_d}},\br{q_{\M^*(i_1)},\cdots,q_{\M^*(i_d)}}\right).
\]
By iterating over every $j_1,\cdots,j_d \in [n]$, when $j_1 = \M^*(i_1),\cdots,j_d = \M^*(i_d)$, it holds that
\begin{equation} \label{RTProp}
(R,t) \in \algname\left(\br{p_{i_1},\cdots,p_{i_d}},\br{q_{j_1},\cdots,q_{j_d}}\right).
\end{equation}

In Lines~\ref{lineESmatching} of Algorithm~\ref{AlgMatching} we iterate over every set of $2d$ indices $i_1,\cdots,i_d,j_1,\cdots,j_d\in [n]$, define the subsets $P' = \br{p_{i_1},\cdots,p_{id}}, Q' = \br{q_{i_1},\cdots,q_{i_d}}$ in Line~\ref{lineDefP2}, and make a call to \algname$(P',Q')$. We then add the returned set of alignments $M'$ to $M$. Hence, due to~\ref{RTProp}, it is guaranteed that the alignment $(R,t)$ that satisfies~\eqref{eq:Rtnomatching} will be added to $M$ in one of the iterations.

Now, after recovering the alignment, we apply it to the set $P$, and we are left with computing the correspondence between the transformed $P$ and $Q$. Observe that, since $(R,t)$ have already been recovered, solving for the optimal correspondence is now trivial. We can compute, for every transformed point in $P$, its nearest neighbor in $Q$. Given the alignment $(R,t)$, this correspondence is the optimal correspondence for the given cost function.

Let $\M = \NN(P,Q,(R,t))$ be the nearest neighbor matching between the transformed $P$ and $Q$. Since $\M$ is an optimal matching function for $P$, $Q$, the alignment $(R,t)$, and the function $\cost$, it satisfies that
\begin{equation} \label{eq:optMatching}
\cost\left(P_{[\M]},Q,(R,t)\right) \leq \cost\left(P_{[\M^*]},Q,(R,t)\right).
\end{equation}
Since $(R,t) \in M$ and $\M = \NN(P,Q,(R,t))$, we obtain that $(R,t,\M) \in S$, where $S$ is the set defined in Line~\ref{lineMappingFunc} of Algorithm~\ref{AlgMatching}. Combining $(R,t,\M) \in S$ with the definition of $(\tilde{R},\tilde{t},\tilde{\M})$ in Line~\ref{lineMinS}, it holds that
\begin{equation} \label{eq:minS}
\cost\left(P_{[\tilde{\M}]},Q,(\tilde{R},\tilde{t})\right) \leq \cost\left(P_{[\M]},Q,(R,t)\right).
\end{equation}

Hence, the following holds
\[
\begin{split}
& \cost\left(P_{[\tilde{\M}]},Q,(\tilde{R},\tilde{t})\right) \leq \cost\left(P_{[\M]},Q,(R,t)\right)\\
& \leq \cost\left(P_{[\M^*]},Q,(R,t)\right)\\
& \leq w^r (1+\sqrt{2})^{dr} \cdot \cost\left(P_{[\M^*]},Q,(R^*,t^*)\right),
\end{split}
\]
where the first derivation holds by~\eqref{eq:minS}, the second derivation holds by~\eqref{eq:optMatching} and the third derivation is by~\eqref{eq:Rtnomatching}.
Furthermore, the running time of the algorithm is $n^{O(d)}$ since there are $n^{O(d)}$ iterations, each iteration takes time independent of $n$ ($d^{O(d)}$ time). Afterwards, we compute the optimal matching $\NN(P,Q,(R,t))$ for every $(R,t)\in M$. There are $n^{O(d)}$ alignments in $M$, and computing such an optimal matching for each alignment takes $n^{O(1)}$ time. Hence, the total running time is $n^{O(d)}$.

\textbf{Constrained correspondence. }Assume we wish to solve the registration problem under constraints on the correspondence function, for example that the correspondence function is a bijection function. Then the computation of the set $M$ remains unchanged. Afterwards, the only change required is to compute, for every $(R',t') \in M$, the optimal bijective function between the transformed $P$ and $Q$, rather than the nearest neighbor correspondence.

Kuhn and Harold suggested in~\cite{kuhn1955hungarian} an algorithm that given the pairwise distances (fitting loss) between two sets of $n$ elements $P$ and $Q$, it finds an assignment for every $p\in P$ to an element $q\in Q$ that minimizes the sum of distances between every assigned pair. This algorithm takes $O(n^3)$ time. We can use this algorithm to compute an optimal matching function $\hat{\M}(P,Q,(R',t'),\cost)$ for every $(R',t') \in M$ in Line~\ref{lineMappingFunc} of Algorithm~\ref{AlgMatching}. The proof above remains unchanged except that the optimal correspondence function $\M^*$ is assumed to be a bijection.
\end{proof}

\section{Run Time Improvements}

\begin{figure}
  \begin{minipage}[c]{0.2\textwidth}
    \includegraphics[scale=0.4]{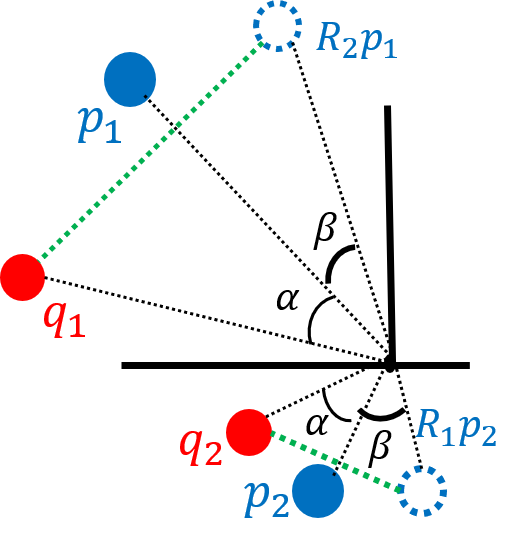}
  \end{minipage}\hfill
  \begin{minipage}[c]{0.27\textwidth}
    \caption{Corresponding pairs $(p_1,q_1)$ and $(p_2,q_2)$, where $\norm{p_1} > \norm{p_2}$. $R_1$ and $R_2$ are rotations that align $(p_1,q_1)$ and $(p_2,q_2)$ respectively. The cost of applying $R_2$ while damaging $(p_1,q_1)$ is bigger than the cost of applying $R_1$ while damaging $(p_2,q_2)$.}\label{fig:normSampling}
  \end{minipage}
\end{figure}

The following claim proves a weak version of the triangle inequality, for the cost functions we define in Definition~\ref{def:cost}.
\begin{claim}[Weak triangle inequality.] \label{claim:weakTria}
Let $z>0$ be a constant. Let $\lip$ be an $r$-log-Lipschitz functions and $\D:\REAL^d\times\REAL^d \to [0,\infty)$ be a function such that $\D(p,q) = \norm{p-q}_z$. Then, for every $p,q,v \in \REAL^d$,
\[
\lip\left(D(p,q)\right) \leq \rho c^r \left(\lip\left(D(p,v)\right) + \lip\left(D(v,q)\right)\right),
\]
where $\rho = \max\br{2^{r-1},1}$ and $c = d^{\abs{\frac{1}{z}-\frac{1}{2}}}$.
\end{claim}
\begin{proof}
The case of $z=2$ immediately holds by substituting $\tilde{D} = \lip$ and $\dist = D$ in Lemma $2.1$ of~\cite{feldman2012data}.
The case $z=2$ can be trivially extended to any constant $z>0$ by combining the following property of vector norms:
For every vector $u\in \REAL^d$ and constant $a>0$ it holds that
\[
\begin{cases}
\norm{u}_2 \leq \norm{u}_a \leq c\norm{u}_2 & \text{if } a\in (0,2]\\
\norm{u}_a \leq \norm{u}_2 \leq c\norm{u}_a & \mbox{if } a>2\\
\end{cases}.
\]
\end{proof}

In what follows, for simplicity, we denote by $\cost(P,Q,R)$ the cost $\cost(P,Q,(R,\vec{0}))$ and by $\cost(P,Q) = \cost(P,Q,(I_d,\vec{0}))$.

The following lemma proves that there is an index $j\in [n]$ that can be recovered via non-uniform sampling, and a rotation matrix $R'$ that aligns the directions of $p_j$ and $q_j$, and with high probability also approximates the total initial cost up to some constant factor.
\begin{lemma} \label{lemProbOneRotation}
Put $\sdim \in [d]$ and $z>0$. Let $\pi$ be an $\sdim$-dimensional subspace of $\REAL^d$, and $P = \br{p_1,\cdots,p_n} \subset \pi$ and $Q = \br{q_1,\cdots,q_n} \subset \pi$ be two ordered sets of points. Let $\cost$ be as defined in Definition~\ref{def:cost} for $f= \norm{v}_1$, an $r$-log-Lipschitz function $\lip$ and $D(p,q) = \norm{p-q}_z$.
Let $R^* \in \rot(d)$ and let $j\in [n]$ be an index sampled randomly, where $j = i$ with probability $w_i = \frac{\norm{p_i}^r}{\sum_{j \in [n]}\norm{p_j}^r}$ if $\norm{q_i} \neq 0$ and $w_i = 0$ otherwise.
Then there is $R' \in \R_\pi$ that satisfy the following properties:
\renewcommand{\labelenumi}{(\roman{enumi})}
\begin{enumerate}
\item $R'p_j \in \linspan(q_j)$.
\item $\cost(P,Q,(R',\vec{0})) \leq 6\rho^2 c^{2r} \cdot \cost(P,Q,(R^*,\vec{0}))$ with probability at least $1/2$, where $\rho = \max\br{2^{r-1},1}$ and $c = d^{\abs{\frac{1}{z}-\frac{1}{2}}}$.
\end{enumerate}
\end{lemma}
\begin{proof}
Without loss of generality assume that $R^*$ is the identity matrix and that $\pi$ is spanned by $e_1,\cdots,e_\sdim$, otherwise rotate the set $P$ of points and rotate the coordinates system respectively.
Furthermore, we remove all pairs $(p_i,q_i)$ of corresponding points from $P$ and $Q$ where $\norm{p_i} = 0$ or $\norm{q_i} = 0$. The distance $D(p_i,q_i)$ between such pairs is unaffected by a rotation of $p_i$, i.e., $D(Rp_i,q_i) = D(p_i,q_i)$ for every rotation matrix $R$. We can therefore ignore such pairs. The sampling probabilities of other pairs will not be affected by removing such $(p_i,q_i)$ since $w_i = 0$ by definition.

For every $i\in [n]$, let $R_i \in \R_{\linspan(\br{p_i,q_i})}$ be a rotation matrix that satisfies $q_i \in \linspan(R_ip_i)$, i.e., aligns the directions of the vectors $p_i$ and $q_i$ by a rotation in the $2$-dimensional subspace (plane) that those two vectors span. If there is more than one such rotation matrix, pick the one that rotates $p_i$ with the smallest angle of rotation.

Now, by the definition of $R_j$ we have that
\begin{equation} \label{triaInequalityLipD}
\lip(D(R_jp_j,q_j)) \leq \lip(D(p_j,q_j))
\end{equation}
for every $j\in [n]$. Therefore, the following holds
\begin{equation} \label{eq:triaForCost}
\begin{split}
\sum_{j=1}^{n}\lip(D(R_jp_j,p_j)) & \leq
\sum_{j=1}^{n}\rho c^r(\lip(D(R_jp_j,q_j)) + \lip(D(q_j,p_j)))\\
& \leq 2\rho c^r\sum_{j=1}^{n} \lip(D(q_j,p_j))\\
& = 2\rho c^r\cost(P,Q),
\end{split}
\end{equation}
where the first derivation is by Claim~\ref{claim:weakTria} and the second derivation is by~\eqref{triaInequalityLipD}.

We now prove that if we sample an index $j \in [n]$ according to the distribution $w = (w_1,\cdots,w_n)$, then the expected cost $\cost(P,Q,R_j)$ between the points of $P$ after a rotation by $R_j$ and their correspond points in $Q$ is at most a constant times the original cost $\cost(P,Q)$.
We observe that
\begin{align}
& \sum_{j\in [n]} w_j \cdot \cost(P,Q,R_j) =
\sum_{j\in [n]} w_j \sum_{i=1}^n\lip(D(R_jp_i,q_i)\\
& \leq \sum_{j\in [n]} w_j \sum_{i=1}^n 2\rho c^r \left(\lip(D(R_jp_i,p_i))+\lip(D(p_i,q_i))\right) \label{eqprobMain1}\\
& = \rho c^r\sum_{j\in [n]} w_j \sum_{i=1}^n\lip(D(R_jp_i,p_i))+ \sum_{j\in [n]} w_j \sum_{i=1}^n\lip(D(p_i,q_i)) \nonumber\\
& = \rho c^r\sum_{j\in [n]} w_j \sum_{i=1}^n\lip(D(R_jp_i,p_i))+ \sum_{i=1}^n\lip(D(p_i,q_i)) \label{eqprobMain2}\\
& = \rho c^r\sum_{j\in [n]} w_j \sum_{i=1}^n\lip(D(R_jp_i,p_i))+ \cost(P,Q) \label{eqprobMain25},
\end{align}
where~\eqref{eqprobMain1} is by substituting in Claim~\ref{claim:weakTria} and~\eqref{eqprobMain2} holds since $w$ is a distribution vector.

We now bound the leftmost term of~\eqref{eqprobMain25}.
\begin{align}
& \sum_{j\in [n]} w_j \sum_{i=1}^n\lip(D(R_jp_i,p_i))\nonumber \\
& = \sum_{j\in [n]} w_j \sum_{i=1}^n\lip\left(\norm{p_i}D\left(\frac{R_jp_i}{\norm{p_i}},\frac{p_i}{\norm{p_i}}\right)\right)\label{eqprobMain3}\\
& \leq \sum_{j\in [n]} w_j \sum_{i=1}^n\norm{p_i}^r\lip\left(D\left(\frac{R_jp_i}{\norm{p_i}},\frac{p_i}{\norm{p_i}}\right)\right)\label{eqprobMain4}\\
& = \sum_{j\in [n]} \frac{\norm{p_j}^r}{\sum_{k \in [n]}\norm{p_k}^r} \sum_{i=1}^n\norm{p_i}^r\lip\left(D\left(\frac{R_jp_i}{\norm{p_i}},\frac{p_i}{\norm{p_i}}\right)\right)\label{eqprobMain5}\\
& \leq \sum_{j\in [n]} \frac{\norm{p_j}^r}{\sum_{k \in [n]}\norm{p_k}^r} \sum_{i=1}^n\norm{p_i}^r\lip\left(D\left(\frac{R_jp_j}{\norm{p_j}},\frac{p_j}{\norm{p_j}}\right)\right)\label{eqprobMain5_2}\\
& = \sum_{j\in [n]} \norm{p_j}^r\lip\left(D\left(\frac{R_jp_j}{\norm{p_j}},\frac{p_j}{\norm{p_j}}\right)\right)\nonumber\\
& \leq \sum_{j\in [n]} \lip\left(D\left(R_jp_j,p_j\right)\right)\label{eqprobMain7}\\
& = 2\rho c^r\cost(P,Q) \label{eqprobMain8}
\end{align}
where~\eqref{eqprobMain3} holds since $D$ is simply a norm function, ~\eqref{eqprobMain4} holds since $\lip$ is an $r$-log Lipschitz function, ~\eqref{eqprobMain5} is simply by substituting the value of $w_j$,  \eqref{eqprobMain5_2} holds by Claim~\ref{claimRot}, ~\ref{eqprobMain7} holds since $D$ is a norm and $\lip$ is an $r$-log Lipschitz function, and~\eqref{eqprobMain8} is by~\eqref{eq:triaForCost}.

Combining~\eqref{eqprobMain8} and~\eqref{eqprobMain25} yields that
\begin{equation} \label{eq:costExpectation}
\sum_{j\in [n]} w_j \cdot \cost(P,Q,R_j) \leq 3\rho^2 c^{2r} \cdot \cost(P,Q).
\end{equation}

For a random variable $X$, a positive constant $a > 0$, the Markov inequality states that
\begin{equation} \label{eq:markov}
Pr(X \geq a) \leq \frac{E(X)}{a},
\end{equation}
where $E(X)$ is the expectation of $X$.

Define the random variable $X := \cost(P,Q,R_j)$, where the randomness is over the choice of the index $j\in [n]$, and let $a = 6\rho^2 c^{2r} \cdot \cost(P,Q)$. By~\eqref{eq:costExpectation}, the expectation of $X$ is $E(X) = 3\rho^2 c^{2r} \cdot \cost(P,Q)$. Plugging into the Markov equality~\eqref{eq:markov} yields that
\[
\cost(P,Q,R_j) \leq 6\rho^2 c^{2r} \cdot \cost(P,Q).
\]
with probability at least $1/2$.
\end{proof}

The following lemma proves the correctness of Algorithm~\ref{alg:probrot}.
\begin{lemma} [Lemma~\ref{lemProbRotation}] \label{lemProbRotation_proof}
Let $P = \br{p_1,\cdots,p_n}$ and $Q = \br{q_1,\cdots,q_n}$ be two ordered sets of points in $\REAL^d$ and let $z>0$. 
Let $\cost$ be as defined in Definition~\ref{def:cost} for $f=\norm{v}_1$, some $r$-log Lipschitz $\lip$ and $D(p,q) = \norm{p-q}_z$.
Let $R'$ be an output of a call to \algNameProbRot$(P,Q,r)$; see Algorithm~\ref{alg:probrot}. Then, with probability at least $\frac{1}{2^{d-1}}$,
\begin{equation} \label{eqToProveProbRot}
\cost(P,Q,(R',\vec{0})) \leq \sigma \cdot \min_{R \in \rot(d)} \cost(P,Q,(R,\vec{0})),
\end{equation}
for a constant $\sigma$ that depends on $d$, $r$ and $z$. Furthermore, $R'$ is computed in $O(nd^2)$ time.
\end{lemma}
\begin{proof}
Let
\[
R^* \in \argmin_{R \in \rot(d)} \cost(P,Q,(R,\vec{0}))
\]
be the optimal rotation matrix, and let $\rho = \max\br{2^{r-1},1}$ and $c = d^{\abs{\frac{1}{z}-\frac{1}{2}}}$.


Let $j_1 \in [n]$ be an index sampled randomly, where $j = i$ with probability $w_i = \frac{\norm{p_i}^r}{\sum_{j \in [n]}\norm{p_j}^r}$ if $\norm{q_i} \neq 0$ and $w_i = 0$ otherwise. Observe that the probabilities $w_i$ are independent of $R^*$ since a rotation matrix does not change the norms of the points, i.e., $\norm{R^*p_i} = \norm{p_i}$ for every $i\in [n]$.
By Lemma~\ref{lemProbOneRotation}, there exists a matrix $R_1$ that aligns the direction vectors of $p_{j_1}$ and $q_{j_1}$, and with probability at least $1/2$ satisfies:
\begin{equation} \label{eq:rotOptToRot1}
\cost(P,Q,R_1R^*) \leq 6\rho^2 c^{2r} \cdot \cost(P,Q,R^*).
\end{equation}
However, there might be an infinite set $A_1$ of such rotation matrices which align the direction vectors of $p_{j_1}$ and $q_{j_1}$.
Let $R_1 \in A_1$ be an arbitrary such rotation matrix.

Let $P'$ be the set $P$ after applying $R_1R^*$, and let $\hat{P}'$ and $\hat{Q}$ be the sets $P'$ and $Q$ respectively after projecting their points onto the hyperplane orthogonal to $q_{j_1}$ i.e.,
\[
P' = \br{p'_i:=R_1 R^{*}p_i \mid i \in [n]},
\]
\[
\hat{P}' = \br{\hat{p}_i:=WW^Tp_i' \mid i \in [n]},
\]
and
\[
\hat{Q} = \br{\hat{q}_i:=WW^Tq_i \mid i \in [n]},
\]
where $W \in \REAL^{d\times (d-1)}$ is an orthogonal matrix whose columns span the hyperplane $H$ orthogonal to $q_{j_1}$.

Let $j_2 \in [n]$ be an index sampled randomly, where $j = i$ with probability $w_i = \frac{\norm{\hat{p}_i}^r}{\sum_{j \in [n]}\norm{\hat{p}_j}^r}$ if $\norm{\hat{q}_i} \neq 0$ and $w_i = 0$ otherwise.
By applying Lemma~\ref{lemProbOneRotation} again on using $\hat{P}'$ and $\hat{Q}$, there is a matrix $R_2$ that aligns the direction vectors of $\hat{p}_{j_2}$ and $\hat{q}_{j_2}$, and with probability at least $1/2$ satisfies:
\begin{equation} \label{eq:R2Prop}
\cost(\hat{P}',\hat{Q},R_2) \leq 6\rho^2 c^{2r} \cdot \cost(\hat{P}',\hat{Q},I_d).
\end{equation}
However, again, there might be an infinite set $A_2$ of such rotation matrices which align the direction vectors of $\hat{p}_{j_2}$ and $\hat{q}_{j_2}$.
Let $R_2 \in A_2$ be an arbitrary such rotation matrix.

We now prove the following claims: (i) the cost $\cost(P',Q,R_2)$ of applying the rotation matrix $R_2$ to the (unprojected) sets $P'$ and $Q$ will approximate the cost $\cost(P',Q,I_d)$,
(ii): the choice of $j_2$ is independent of the choice of $R_1$, and (iii) the vectors $q_{j_1}$ and $\hat{q}_{j_2}$ are orthogonal.

\begin{claim}
It holds that
\begin{equation} \label{claim:rot1ToRot2}
\cost(P,Q,R_2R_1R^*) \leq 12 \rho^4 c^{5r} \cost(P,Q,R_1R^*).
\end{equation}
\end{claim}
\begin{proof}
Recall that $P' = \br{p_1',\cdots,p_n'}$. Consider the hyperplane $H$ orthogonal to $q_{j_1}$ (the hyperplane the points are projected on after the first step). Let $H_i$ be a hyperplane parallel to $H$ but passes through $q_i$. Let $v_i = \proj(R_2p'_i, H_i)$ be the projection of $R_2p_i'$ onto the hyperplane $H_i$ for every $i\in [n]$.

Observe that, by construction, the rotation matrix $R_2$ rotates every point $p_i'$ around the rotation axis $q_{j_1}$, which is orthogonal to $H_i$. Therefore, the distance between $R_2p_i'$ to its (orthogonal) projection onto $H_i$ equals the distance between $p'_i$ and its (orthogonal) projection onto $H_i$. Formally,
\begin{equation} \label{eqProbRot3_proof}
\begin{split}
D(R_2p_i', \proj(R_2p_i', H_i)) = D(p_i', \proj(p_i',H_i)).
\end{split}
\end{equation}

Let $v_i = \proj(R_2p'_i, H_i)$ be the projection of $R_2p'_i$ onto the hyperplane $H_i$ for every $i\in [n]$. We now have that
\begin{align}
& \cost(P',Q,R_2)
= \sum_{i=1}^n \lip\left(D(R_2p'_i, q_i)\right) \label{eqProbRot1}\\
& \leq \rho c^r \sum_{i=1}^n \lip\left(D(R_2p'_i, v_i)\right) +  \rho c^r \sum_{i=1}^n\lip\left(D(v_i, q_i)\right) \label{eqProbRot2}\\
& = \rho c^r \sum_{i=1}^n \lip\left(D(p'_i, \proj(p'_i,H_i))\right) +  \rho c^r \sum_{i=1}^n\lip\left(D(v_i, q_i)\right) \label{eqProbRot3}
\end{align}
where~\eqref{eqProbRot1} is by the definition of $\cost$, \eqref{eqProbRot2} is by the weak triangle inequality in Claim~\ref{claim:weakTria}, and \eqref{eqProbRot3} is by combining~\eqref{eqProbRot3_proof} with the definition of $v_i$.

We now bound the rightmost term of~\eqref{eqProbRot3} as follows:
\begin{align}
& \rho c^r \sum_{i=1}^n\lip\left(D(v_i, q_i)\right) \nonumber\\
& = \rho c^r \sum_{i=1}^n\lip\left(D(\proj(v_i,H), \proj(q_i,H))\right) \label{eqProbRot4}\\
& = \rho c^r \cdot \cost(\hat{P}',\hat{Q},R_2) \label{eqProbRot4_2}\\
& \leq 6 \rho^3 c^{3r} \cdot \cost(\hat{P}',\hat{Q},I_d) \label{eqProbRot4_3}\\
& = 6 \rho^3 c^{3r} \cdot \sum_{i=1}^n\lip\left(D(\proj(p'_i,H), \proj(q_i,H))\right) \label{eqProbRot4_4}\\
& = 6 \rho^3 c^{3r} \sum_{i=1}^n\lip\left(D(\proj(p'_i,H_i), q_i)\right), \label{eqProbRot6}
\end{align}
where~\eqref{eqProbRot4} holds by combining that $v_i, q_i \in H_i$ and that $H_i$ and $H$ are two parallel hyperplanes, \eqref{eqProbRot4_2} holds by the definitions of $v_i$ and $R_2$, \eqref{eqProbRot4_3} is by~\eqref{eq:R2Prop}, \eqref{eqProbRot4_4} is by the definition of $\hat{P}'$ and $\hat{Q}$, and~\eqref{eqProbRot6} holds by combining that $q_i \in H_i$ and that $H$ and $H_i$ are parallel.

Now, consider the triangle $\Delta(p'_i, \proj(p'_i, H_i),q_i)$. This triangle is a right triangle since $q_i \in H_i$. Hence, $D_2(p'_i, \proj(p'_i, H_i)) \leq D_2(p'_i,q_i)$ and $D_2(\proj(p'_i, H_i), q_i) \leq D_2(p'_i,q_i)$. By the properties of vector norms and since $\lip$ is an $r$-log-Lipschitz function, we obtain that
\begin{equation} \label{eqRightTriangleDz}
\begin{split}
& \lip\left(D(p'_i, \proj(p'_i, H_i))\right) + \lip\left(D(\proj(p'_i, H_i), q_i)\right)\\
& \leq 2c^r\lip\left(D(p'_i, q_i)\right).
\end{split}
\end{equation}

Combining the above yields that
\begin{align}
& \cost(P,Q,R_2R_1R^*) = \cost(P',Q,R_2) \label{eqProbRot7_0}\\
& \leq \rho c^r \sum_{i=1}^n \lip\left(D(p'_i, \proj(p'_i,H_i))\right) +  \rho c^r \sum_{i=1}^n\lip\left(D(v_i, q_i)\right) \label{eqProbRot7}\\
& \leq \rho c^r \sum_{i=1}^n \lip\left(D(p'_i, \proj(p'_i,H_i))\right) \nonumber\\
& \quad +  6 \rho^4 c^{4r} \sum_{i=1}^n\lip\left(D(\proj(p'_i,H_i), q_i)\right) \label{eqProbRot8}\\
&\leq 6 \rho^4 c^{4r} \sum_{i=1}^n \lip\left(D(p'_i, \proj(p'_i,H_i))\right) \nonumber\\
& \quad +  6 \rho^4 c^{4r} \sum_{i=1}^n\lip\left(D(\proj(p'_i,H_i), q_i)\right) \nonumber\\
&\leq 12 \rho^4 c^{5r} \sum_{i=1}^n \left( \lip\left(D(p'_i, q_i)\right)\right) \label{eqProbRot10}\\
& = 12 \rho^4 c^{5r} \cdot \cost(P',Q,I_d) \label{eqProbRot11_0}\\
& = 12 \rho^4 c^{5r} \cdot \cost(P,Q,R_1R^*), \label{eqProbRot11}
\end{align}
where~\eqref{eqProbRot7_0} is by the definition of $P
$, \eqref{eqProbRot7} is by~\eqref{eqProbRot3}, \eqref{eqProbRot8} is by~\eqref{eqProbRot6}, \eqref{eqProbRot10} is by~\eqref{eqRightTriangleDz}, and~\eqref{eqProbRot11} is by the definition of $\cost$.

\end{proof}

\begin{claim} \label{claim:j2Independent}
The choice of $j_2$ is independent of the choice of $R_1$.
\end{claim}
\begin{proof}
The choice of $j_2$ depends on the $\ell_2$ norms $\norm{\hat{p}_i}$ of the points in $\hat{P}'$. Observe that the rotation matrices in the set $A_1$ rotate the set $P$ around the axis $q_{j_1}$, and that the projection matrix $W_{j_1}W_{j_1}^T$ projects any point onto the hyperplane orthogonal to $q_{j_1}$. Therefore, for any two matrices $M_1,M_2 \in A_1$, the norms of the vectors $WW^TM_1p$ and $WW^TM_2p$ are the same, for every $p\in P'$. Hence, the distribution from which $j_2$ is drawn is identical for any two matrices in $A_1$.
\end{proof}

\begin{claim}
The vectors $q_{j_1}$ and $\hat{q}_{j_2}$ are orthogonal.
\end{claim}
\begin{proof}
By construction, the vector $\hat{q}_{j_2}$ is obtained by a projection of $q_{j_2} \in Q$ onto the subspace orthogonal to $q_{j_1}$. Therefore, they are orthogonal.
\end{proof}

The above claims prove the existence of $2$ indices, $j_1$ and $j_2$, and two rotation matrices $R_1$ and $R_2$ that align the direction vectors of $p_{j_1}$ with $q_{j_1}$ and $\hat{p}_{j_2}$ with $\hat{q}_{j_2}$ respectively. The choice of $j_1$ and $j_2$ is independent of any initial rotation $R^*$ of $P$ and independent of the choice of $R_1$ respectively. Consider the rotation matrix $\hat{R} = R_2R_1R^*$. Since $q_{j_1}$ and $\hat{q}_{j_2}$ are orthogonal, then $\hat{R}$ can simultaneously align both pairs of vectors, i.e., $\hat{R}$ satisfies both constraints. By combining~\eqref{claim:rot1ToRot2} and~\eqref{eq:rotOptToRot1} we obtain that with probability at least $1/4$,
\[
\cost(P,Q,\hat{R}) \leq (12 \rho^4 c^{5r})^2 \cost(P,Q,R^*).
\]

Repeating the above steps $d-1$ times yields that there is a set of $d-1$ indices $j_1,\cdots,j_{d-1}$, corresponding rotation matrices $R^{(1)},\cdots,R^{(d-1)}$, and a rotation matrix $R' = R_{d-1} \cdots R_1R^*$ that satisfies
\renewcommand{\labelenumi}{(\roman{enumi})}
\begin{enumerate}[leftmargin=0cm]
\item $R'$ aligns the direction vectors of $p_{j_1}$ and $q_{j_1}$, i.e., $R'p_{j_1} \in \linspan(q_{j_1})$.
\item For every $i\in \br{2,\cdots,d-1}$, $R'$ aligns the direction vectors of $p_{j_i}$ and $q_{j_i}$ after their projection onto the hyperplane orthogonal to $q_{j_1}$, then onto the hyperplane orthogonal to $q_{j_2}$, and so on until the projection onto the hyperplane orthogonal to $q_{j_{i-1}}$.
\item The indices $j_1,\cdots,j_{d-1}$ are independent of the initial $R^*$, and also independent of the arbitrary choice of rotation matrices throughout the $d-1$ steps.
\item Identically to~\ref{claim:rot1ToRot2}, for every $k\in \br{2,\cdots,d-1}$ we can prove that with probability at least $1/2$,
\[
\cost(P,Q,R_k\cdots R_1R^*) \leq 12 \rho^4 c^{5r} \cdot \cost(P,Q,R_{k-1}\cdots R_1R^*).
\]
By combining the last inequality for every $k\in \br{1,\cdots,d-1}$, we obtain that with probability at least $1/2^{d-1}$,
\[
\begin{split}
\cost(P,Q,R') & \leq (12 \rho^4 c^{5r})^{d-1} \cost(P,Q,R^*).
\end{split}
\]
\end{enumerate}

Furthermore, since the pair of aligned direction vectors at the $i$'th step are orthogonal to all the previous $i-1$ aligned direction vectors by construction, the obtained set of $d-1$ constraints on the output rotation are Linearly independent. They thus determine a single rotation matrix $R'$, which can be recovered without knowing either, $R^*$, $R_1,\cdots, R_{d-1}$.
Therefore, those indices can be recovered iteratively as the above constructive proof suggests, and $R'$ can be then recovered from the constraints those pairs determine.

Algorithm~\ref{alg:probrot} is a recursive implementation of the above $d-1$ steps. It computes the set of indices $j_1,\cdots,j_{d-1}$ discussed above, and returns the rotation matrix $R'$ that they determine.

There are at most $d-1$ recursive iterations. Each iteration takes at most $O(nd)$ time. Therefore, the total running time is $O(nd^2)$.
\end{proof}

\begin{theorem} [Theorem~\ref{lemProbAlignment}] \label{lemProbAlignment_proof}
Let $P = \br{p_1,\cdots,p_n}$ and $Q = \br{q_1,\cdots,q_n}$ be two ordered sets of points in $\REAL^d$ and let $z>0$.
Let $\cost$ be as defined in Definition~\ref{def:cost} for $f=\norm{v}_1$, some $r$-log Lipschitz $\lip$ and $D(p,q) = \norm{p-q}_z$.
Let $M$ be an output of a call to \algNameApproxAlignment$(P,Q,r)$; see Algorithm~\ref{alg:prob}. Then, with constant probability greater than $1/2$, there is an alignment $(R',t') \in M$ that satisfies
\begin{equation}
\cost(P,Q,(R',t')) \leq \sigma\cdot \min_{(R,t) \in \alignments} \cost(P,Q,(R,t)),
\end{equation} \label{eqProbToProve}
for a constant $\sigma$ that depends on $d$, $r$ and $z$. Furthermore, $M$ is computed in $O\left(\frac{nd^2}{\log\left(\frac{2^d}{2^d-1}\right)}\right)$ time.
\end{theorem}
\begin{proof}
Let $(R^*,t^*)$ be the optimal alignment, i.e.,
\[
(R^*,t^*) \in \argmin_{(R,t) \in \alignments} \cost(P,Q,(R,t)).
\]
Throughout this proof, for simplicity of notation, we assume that the points of $P$ have already been rotated and translated by the alignment $(R^*,t^*)$, i.e., we assume that $R^*$ is the identity matrix and $t^*$ is a zeros vector. Therefore, the following proof is only a prove of existence, since $P$ is not known in practice. At the end of the proof, we explain how to bridge this gap.

Let $\OPT$ be the minimal cost, i.e.,
\begin{equation} \label{optDef}
\OPT = \cost(P,Q,(I_d,\vec{0})) = \sum_{i=1}^n \lip\left(D(p_i, q_i)\right).
\end{equation}

\begin{claim} \label{claim:HalfAreGood}
There are at least $n/2$ corresponding pairs $(p_i,q_i)$ that satisfy
\[
\lip\left(D(p_i, q_i)\right) \leq \frac{2\OPT}{n}.
\]
\end{claim}
\begin{proof}
Falsely assume that there are less than $n/2$ such pairs. This implies that there are at least $n/2$ pairs which satisfy
\[
\lip\left(D(p_i, q_i)\right) > \frac{2\OPT}{n}.
\]
The cost of those (at least) $n/2$ pairs would thus be greater than $\frac{n}{2} \cdot \frac{2\OPT}{n} = \OPT$, which contradicts the definition of $\OPT$ as the cost of the whole $n$ pairs.
\end{proof}

We now prove that translating the set $P$ by $(p_j-q_j)$ for a randomly selected index $j\in [n]$ yields a constant factor approximation to $\OPT$.
Let $j\in [n]$ be an index selected uniformly at random, let $\rho = \max\br{2^{r-1},1}$ and $c = d^{\abs{\frac{1}{z} - \frac{1}{2}}}$. Then,
\begin{align}
& \cost(P,Q,(I_d,(p_j-q_j)))
= \sum_{i=1}^n \lip\left(D(p_i-(p_j-q_j), q_i)\right) \label{eqProbT1}\\
& \leq \rho c^r \left( \sum_{i=1}^n \left(\lip\left(D(p_i-(p_j-q_j), p_i)\right) + \lip\left(D(p_i, q_i) \right)\right) \right) \label{eqProbT2}\\
& = \rho c^r \left( \sum_{i=1}^n \lip\left(D(p_j,q_j)\right) + \sum_{i=1}^n\lip\left(D(p_i, q_i)\right) \right) \nonumber\\
& \leq \rho c^r \left( \sum_{i=1}^n \frac{2\OPT}{n} + \sum_{i=1}^n \lip\left(D(p_i, q_i)\right) \right) \label{eqProbT4}\\
& = \rho  c^r \left(2\OPT+ \sum_{i=1}^n \lip\left(D(p_i, q_i)\right) \right) \nonumber\\
& = 3\rho  c^r \cdot \OPT, \label{eqProbT6}
\end{align}
where~\eqref{eqProbT1} is by the definition of $\cost$, \eqref{eqProbT2} is by the weak triangle inequality in Claim~\ref{claim:weakTria}, \eqref{eqProbT4} holds with probability at least $1/2$ by combining Claim~\ref{claim:HalfAreGood} with the random pick of the index $j$, and~\eqref{eqProbT6} holds by the definition of $\OPT$ in~\eqref{optDef}.

Therefore, a translation of $P$ by $t' = p_j-q_j$ where $j\in[n]$ is selected uniformly at random yields a constant factor approximation to $\OPT$.

We now consider the sets $P' = \br{p-(p_j-q_j) - q_j \mid p \in P} = \br{p-p_j \mid p \in P}$ and $Q' = \br{q - q_j \mid q \in Q}$, i.e., we translate both set, the (already translated) set $P$ and $Q$, such that $q_j$ is the origin. Observe that translating both sets by the same vector $q_j$ does not change the cost.

Let $R'$ be the output of a call to \algNameProbRot$(P',Q',r)$; see Algorithm~\ref{alg:probrot}. Then by Lemma~\ref{lemProbRotation_proof}, we obtain that, with probability at least $\frac{1}{2^{d-1}}$,
\begin{equation} \label{eq:Rtagapproxopt}
\begin{split}
& \sum_{i=1}^n \lip\left(D(R'(p_i-p_j), q_i-q_j)\right) = \cost(P',Q',(R',\vec{0}))\\
& \leq \sigma' \cdot \min_{R \in \rot(d)} \cost(P',Q',(R,\vec{0}))\\
& = \sigma' \cdot \min_{R \in \rot(d)} \sum_{i=1}^n \lip\left(D(R(p_i-p_j), q_i-q_j)\right),
\end{split}
\end{equation}
for a constant $\sigma'$ that depends on $d$, $r$ and $z$, and can be inferred from the proof of Lemma~\ref{lemProbRotation_proof}.
We therefore obtain that
\begin{align}
& \cost(P,Q, (R',R'p_j-q_j)) \\
& = \sum_{i=1}^n \lip\left(D(R'p_i-(R'p_j-q_j), q_i)\right) \label{eqPropAlignment1}\\
& = \sum_{i=1}^n \lip\left(D(R'p_i-R'p_j, q_i-q_j)\right) \label{eqPropAlignment2}\\
& = \sum_{i=1}^n \lip\left(D(R'(p_i-p_j), q_i-q_j)\right) \nonumber\\
& \leq \sigma' \cdot \min_{R \in \rot(d)} \sum_{i=1}^n \lip\left(D(R(p_i-p_j), q_i-q_j)\right) \label{eqPropAlignment4}\\
& \leq \sigma' \cdot \sum_{i=1}^n \lip\left(D(p_i-p_j, q_i-q_j)\right) \nonumber\\
& = \sigma' \cdot \sum_{i=1}^n \lip\left(D(p_i-(p_j-q_j), q_i)\right) \label{eqPropAlignment6}\\
& \leq 3\rho  c^r \sigma' \cdot \OPT \label{eqPropAlignment7},
\end{align}
where~\eqref{eqPropAlignment1} is by the definition of $\cost$, \eqref{eqPropAlignment2} holds since translating both sets by the same vector does not change the cost, \eqref{eqPropAlignment4} holds with probability $\frac{1}{2^{d-1}}$ by~\eqref{eq:Rtagapproxopt}, \eqref{eqPropAlignment6} holds since translating both sets by the same vector does not change the cost, and~\eqref{eqPropAlignment7} holds with probability $1/2$ by~\eqref{eqProbT6}.

Therefore, the rotation matrix $R'$ and translation vector $R'p_j-q_j$ satisfy~\eqref{eqProbToProve}, for $\sigma = 3\rho  c^r \sigma'$, with probability, at least $\frac{1}{2^{d}}$.

\paragraph{Bridging the gap.} As assumed throughout this proof, the optimal rotation $R^*$ is assumed to be the identity matrix, and the optimal translation vector $t^*$ is assumed to be the zeros vector. Therefore, we need to show how to compute, in practice, the approximated alignment $(R',R'p_j-q_j)$ above.

\textbf{Computing $j$. }The index $j$ can be simply computed by a uniform random sample from $\br{1,\cdots,n}$.

\textbf{Computing $R'$. }As explained above, $R'$ is the output of a call to \algNameProbRot$(P',Q',r)$, where $Q'$ is simple the translation of $Q$ such that $q_j$ intersects the origin, and $r$ is given. On the other hand, $P'$ is unknown to us since it depends on the optimal alignment $(R^*,t^*)$. Fortunately, $P'$ satisfies that its $j$'th point intersects the origin. Furthermore, as shown in Lemma~\ref{lemProbRotation_proof}, Algorithm~\ref{alg:probrot} is independent of the rotation of the input set $P'$ in the call \algNameProbRot$(P',Q',r)$.
Therefore, by defining $\hat{P}$ to be a translated version of the (given) set $P$, such that $p_j$ intersects the origin, i.e., $\hat{P} = \br{p-p_j \mid p \in P}$, we get that the call \algNameProbRot$(\hat{P},Q',r)$ is equivalent to the call \algNameProbRot$(P',Q',r)$. Therefore, computing $R'$ is indeed feasible, even though $(R^*,t^*)$ is not known.

Algorithm~\ref{alg:prob} picks an index $j\in [n]$ at random, and computes the above rotation matrix $R'$ using a call to Algorithm~\ref{alg:probrot}. It then appends the alignment $(R',R'p_j-q_j)$ to the output set $M$. This process is repeated $\frac{1}{\log\left(\frac{2^d}{2^d-1}\right)}$ times to amplify the success probability. Therefore, with probability at least $1/2$, the output set $M$ contains an alignment that satisfies~\eqref{eqProbToProve}.

The running time of each iteration is $O(nd^2)$ due to the call to Algorithm~\ref{alg:probrot}. Therefore, the total running of the $\log\left(\frac{2^d}{2^d-1}\right)$ iterations is $O\left(\frac{nd^2}{\log\left(\frac{2^d}{2^d-1}\right)}\right)$.
\end{proof}

\end{document}